%% file: main.tex
\documentclass{article}

\usepackage{microtype}
\usepackage{graphicx}
\usepackage{subcaption}
\usepackage{booktabs} %
\usepackage{amsmath, amsthm, amscd, amsfonts, amssymb, graphicx, bbold, multirow}
\usepackage{mathtools, amsfonts, bm, stmaryrd}
\usepackage{hyperref}
\usepackage{enumitem}

\usepackage[preprint]{neurips_2026}

\PassOptionsToPackage{numbers, compress}{natbib}
\usepackage{algorithm}     
\usepackage[noend]{algorithmic}   

\usepackage{amsmath}
\usepackage{amssymb}
\usepackage{mathtools}
\usepackage{amsthm}
\usepackage{dsfont}
\usepackage[toc]{appendix} 

\usepackage[capitalize,noabbrev]{cleveref}

\theoremstyle{plain}
\newtheorem{theorem}{Theorem}[section]
\newtheorem{proposition}[theorem]{Proposition}

\newtheorem{corollary}[theorem]{Corollary}
\theoremstyle{definition}

\newtheorem{assumption}[theorem]{Assumption}
\theoremstyle{remark}

\newcommand{\E}{\mathbb{E}}

\newcommand{\indep}{\perp \!\!\! \perp}

\newcommand{\new}{}

\usepackage[textsize=tiny]{todonotes}

\title{Principled Federated Random Forests\\ for Heterogeneous Data}

\author{%
  R\'emi Khellaf\\ Inria PreMeDICaL Inserm \\University of Montpellier \\ France
  \And Erwan Scornet \\ Sorbonne Universit\'e \\ Universit\'e Paris Cit\'e \\CNRS \\ LPSM \\F-75005 Paris
  \And Aur\'elien Bellet\\ Inria PreMeDICaL Inserm \\University of Montpellier \\ France
  \And Julie Josse\\ Inria PreMeDICaL Inserm \\University of Montpellier \\ France
}

\begin{document}

\maketitle

\begin{abstract}
    Random Forests (RF) are among the most powerful and widely used predictive models for centralized tabular data, yet few methods exist to adapt them to the federated learning setting. %
    Unlike most federated learning approaches, the piecewise-constant nature of RF prevents exact gradient-based optimization.
    As a result, existing federated RF implementations rely on unprincipled heuristics: for instance, aggregating decision trees trained independently on clients fails to optimize the global impurity criterion, even under simple distribution shifts. 
    We propose FedForest, a new federated RF algorithm for horizontally partitioned data that naturally accommodates diverse forms of client data heterogeneity, from covariate shift to more complex outcome shift mechanisms.
    We prove that our splitting procedure, based on aggregating carefully chosen client statistics, closely approximates the split selected by a centralized algorithm. 
    Moreover, FedForest allows splits on client indicators, 
    enabling a non-parametric form of personalization that is absent from prior federated random forest methods.
    Empirically, we demonstrate that the resulting federated forests closely match centralized performance across heterogeneous benchmarks while remaining communication-efficient.
\end{abstract}

\section{Introduction}\label{sec:intro}
Random Forests (RFs) \citep{breiman2001random} are ensembles of CART decision trees \citep{breiman1984cart} that recursively partition the feature space into regions, referred to as nodes, by selecting splits that maximize an impurity reduction criterion (e.g., variance or Gini). Predictions are then produced by aggregating the outcomes associated with the terminal nodes, commonly called leaves.
Despite the deep learning revolution, RFs remain among the state-of-the-art and most widely used predictors for tabular data \citep{fernandez2014we,DBLP:conf/nips/GrinsztajnOV22,kaggle,10.1371/journal.pone.0301541}, while also offering practical advantages: they naturally handle both continuous and categorical features, 
require minimal hyperparameter tuning, and are substantially more computationally efficient than deep learning models \citep{DBLP:conf/nips/GrinsztajnOV22}.
The strong empirical performance of RFs has been established in centralized settings, where data are assumed to be pooled together. However, this assumption is often unrealistic in many real-world applications, particularly in healthcare, where data are distributed across multiple institutions and cannot be centralized due to regulations, governance constraints, or competitive considerations \citep{rieke2020future,antunes2022federated,nguyen2022federated,xu2021federated}. Federated learning (FL) addresses this challenge by enabling collaborative model training on decentralized data \citep{kairouz2021advances}. In this work, we focus on the standard server--client architecture and on horizontal FL, where clients hold disjoint sets of records described by a common feature space.

Federating random forests presents challenges fundamentally different from those of parametric models, whose training can be federated by aggregating parameters or gradients across clients \citep[see e.g., FedAvg][]{mcmahan2017communication}.
In contrast, RFs are nonparametric and fully data-adaptive: each observation affects not only impurity values at each node but also the set of candidate split thresholds, which depend on the ordered feature values.
Moreover, tree construction is inherently greedy and sequential, relying on hard, discontinuous splits; consequently, no gradients are available for aggregation across clients. To faithfully reproduce centralized behavior, impurity-based split decisions must be computed {and updated} using only aggregated client information, while respecting the sequential dependencies and without ever sharing raw data.
This is made especially challenging by the heterogeneity of data across clients inherent to the federated setting: clients may differ in sample sizes, covariate distributions, or outcome mechanisms, and greedy split selection is highly sensitive to such differences. As a result, naive approaches—such as assembling locally trained forests or aggregating local split scores—can deviate substantially from centralized behavior.

\paragraph*{Related work.}
A recent survey by \citet{wang2024decision} identifies three key challenges in decentralized training of tree-based models: distributed computation of split statistics, limiting communication costs, and handling client heterogeneity.

Decision tree-based federated learning has mostly focus on gradient-boosted decision trees (GBDTs), which federate training by aggregating histogram-based gradient statistics \citep{ke2017lightgbm,li2023fedtree} or by sequentially fitting local trees to the residuals of a global model \citep{roth2022nvidia}. GBDTs coordinate split selection using gradient information, whereas random forests provide no gradients and must rely on impurity-based, discrete split decisions.
Moreover, boosted trees are inherently sequential in their construction: each tree depends on the residuals of all previous ones, and split selection typically requires full synchronization of clients at every node of every tree. In contrast, RFs are bagging-based, allowing trees to be constructed independently and in parallel—an important structural advantage in federated settings.
Finally,
federated GBDTs do not provide theoretical guarantees that the federated procedure reproduces the centralized greedy split objective.

Turning to federated RFs, early work focused on vertically partitioned data, using secure protocols for tree construction \citep{liu2020federated}. In the horizontal setting we consider, most approaches are \emph{local-ensemble} methods: clients train trees or forests locally, and the server either aggregates predictions \citep{liu2022federated} or subsamples local trees to form a global ensemble \citep{hauschild2022federated,cotorobai2025federated,xiang2024transfer,gu2023commute}. These heuristics are best suited to approximately i.i.d.\ data; they neither explicitly accommodate small clients (since each local tree is fit on a reduced per-client sample) nor address covariate shift or client-specific feature–outcome relationships, and are therefore highly sensitive to client heterogeneity.

To the best of our knowledge, no prior work federates the CART split-selection step \emph{itself} in horizontal FL while maintaining a clear link to the centralized greedy objective and accounting for realistic client-level data heterogeneity. The closest approach is \citet{kalloori2022cross}, which federates split selection by aggregating clients’ local histograms at each node. However, their method is heuristic: it approximates centralized CART through binning, offers no theoretical guarantees on split fidelity, and is not designed to handle covariate or outcome shifts, nor outlier data points.

Although RFs enjoy a substantial theoretical foundation in the centralized setting \citep{
biau2016random,scornet2026theory}, analogous guarantees are currently lacking for their federated counterparts. To pave the way toward provably reliable federated forests, we study the fundamental question of whether a federated split-selection mechanism can recover, or closely approximate, the greedy split decisions made by centralized CART.\looseness=-1

\paragraph*{Our contributions.}
We propose a unified framework for federating CART decision trees and random forests that faithfully reproduces centralized training under realistic client heterogeneity without sharing individual data.

First, we introduce in \Cref{subsec:candidate_splits} \textbf{a quantile-based candidate threshold generation scheme} and provide theoretical guarantees that it identifies the candidate closest to the split that would be selected by a centralized greedy CART. By estimating the pooled data distribution, defined as the mixture of the clients, this approach ensures that the threshold selection remains robust to client heterogeneity.

Next, we derive \textbf{an exact split-evaluation rule} (\Cref{sec:split_evaluation}). We show that it reconstructs exactly the impurity gains that would result if all data were pooled, while relying solely on aggregated client statistics, thereby yielding decisions identical to the centralized CART algorithm. The key insight is that standard CART impurities admit additive sufficient statistics, which can be summed across clients to recover the centralized impurity—and hence the gain—exactly. We further design \textbf{a splitting rule on client indices $H$}, which our method allows to treat as a categorical feature with no extra communication. To our knowledge, this is the first method to explicitly support such splits, enabling a nonparametric form of personalization.

We emphasize the effect of heterogeneous data on greedy split selection, showing that impurity gains depend not only on feature-outcome relationships but also on covariate distributions. Consequently, adding a client with a different marginal $P(X)$ can change split decisions even when $P(Y\mid X)$ is shared across clients. 
We also \textbf{propose \emph{AvgImp}, a more efficient approach for homogeneous regimes} (approximately i.i.d.), which aggregates local impurity gains and is supported by a finite-sample error bound.\looseness=-1

Finally, we leverage the federated CART split-decision rule introduced above to develop our theoretically grounded \emph{FedForest} method. Using standard random-forest randomization (client-stratified bootstrapping and feature subsampling), \emph{FedForest} grows trees in parallel from aggregated client statistics alone, \textbf{recovering centralized-level predictive performance} with communication-efficient training.\looseness=-1

We validate these findings through benchmarks on synthetic and real datasets in \Cref{sec:simulations}. Together, our results provide a practical, theory-backed framework for federating nonparametric trees and random forests under heterogeneous data and lay the foundation for federated extensions of other random forest variants, including quantile \citep{meinshausen2006quantile}, survival \citep{ishwaran2008random}, and causal forests \citep{Wager03072018}.

\section{Problem Setting and Background}\label{subsec:problem_setting}

We consider a federated setting involving $K$ clients collecting data described by the same set of features and addressing a supervised regression or classification task. Each client $k \in [K]$ holds a local dataset $\mathcal{D}_k = \{(H_i, X_i^{(k)}, Y_i^{(k)})\}_{i=1}^{n_k}$ consisting of $n_k$ i.i.d.\ samples drawn from a $d$-dimensional distribution ${P}_k$, where the client indicator $H_i=k$ is constant within $\mathcal{D}_k$. The total sample size is $n = \sum_{k=1}^K n_k$. 

Our goal is to learn a predictor $\hat f:\mathcal X\to\mathcal Y$ that performs well on the pooled population represented by the clients, i.e., the mixture distribution $P=\sum_{k=1}^K \rho_k P_k$ with weights $\rho_k=\mathbb P(H=k)\approx n_k/n$. More precisely, we aim to solve the following population optimization problem:
\begin{equation}\label{eq:global_obj}
\textstyle\min_{f\in\mathcal F}\ \mathcal L( f)
= \sum_{k=1}^K \rho_k\,
\mathbb E_{(X,Y)\sim P_k}\left[\ell \big(Y, f(X)\big)\right],
\end{equation}
where $\ell$ is the squared loss for regression and the $0$--$1$ loss for classification.

\subsection{Centralized CART/RF Split Selection}\label{sec:cart_background}
The mixture-risk objective \eqref{eq:global_obj} is an ideal target. CART and Random Forests approximate it through greedy, node-wise optimization of an impurity criterion $I$ (e.g., variance for regression; Gini or entropy for classification) on a finite data sample. In the centralized setting, all data are pooled as $\mathcal D=\bigcup_{k=1}^K \mathcal D_k$. At a node $\nu\subseteq\mathcal X$ of a tree, let $\mathcal S_\nu\subseteq\mathcal D$ be the samples reaching $\nu$. CART selects a \emph{split threshold} $(j,t)$--that is, a {split variable} $j\in[d]$ and a {cutpoint} $t\in\mathbb R$--in two steps: (i) \emph{candidate generation}, where for each feature $j$, the values $\{x_{i,j}\}_{i\in\mathcal S_\nu}$ are sorted, and candidates $t$ are defined as the midpoints between consecutive distinct values; and (ii) \emph{split evaluation}, where each candidate $(j,t)$ is scored by its empirical impurity reduction
\begin{equation}\label{eq:impurity_reduction}
\Delta I(j,t;\mathcal S_\nu)
= I(\mathcal S_\nu)
-\Big(\tfrac{n_L}{n_\nu}I(\mathcal S_L)+\tfrac{n_R}{n_\nu}I(\mathcal S_R)\Big),
\end{equation}
with $\mathcal S_L=\{i\in\mathcal S_\nu: x_{i,j}\le t\}$, $\mathcal S_R=\{i\in\mathcal S_\nu: x_{i,j}>t\}$, $n_L=|\mathcal S_L|$, $n_R=|\mathcal S_R|$, and $n_\nu=n_L+n_R$. 
$I(\mathcal S_\nu)$ denotes the empirical impurity of the samples reaching node $\nu$; for instance, in a regression setting, it stands for the empirical variance of the outcomes of samples falling into $\mathcal S_\nu$. CART then selects the threshold $(j,t)$ that maximizes $\Delta I(j,t;\mathcal S_\nu)$.

Random Forests build an ensemble of CART trees, adding randomness through bootstrap sampling—each tree is trained on a different sample of $\mathcal{D}$—and feature subsampling at each node (typically $\sqrt{d}$ for classification, $d/3$ for regression). Predictions are aggregated across trees by averaging in regression, or by  majority vote in classification. In practice, individual trees are further regularized using standard stopping rules, such as maximum depth, minimum leaf size, or minimum impurity decrease.\looseness=-1 %

In the federated setting, the pooled dataset is never accessible, and clients may only share aggregated summaries with the server. Existing federated tree methods largely restrict splits to $X$ and lack a principled way to split on $H$. Our goal is thus to reproduce centralized split selection from aggregated client statistics while enabling client-aware splits under heterogeneity.

\subsection{Heterogeneity Scenarios}\label{subsec:heterogeneity}
We formalize the data heterogeneity regimes considered in this work and relate each of them to the behavior of a centralized CART-based model trained on pooled data. \Cref{fig:graphical_models} shows the associated graphical models and highlights situations where the client variable $H$ carries predictive information for the outcome $Y$.

\begin{figure}[t]
\centering
\begin{tikzpicture}[node distance={15mm}, thick,
    main/.style = {draw, circle, minimum size=8mm},
    arrow/.style={-stealth},
    lbl/.style={font=\small, midway, fill=white, inner sep=1pt}]
    
    \node[main] (H) {$H$};
    \node[main] (X) [below left=7mm and 8mm of H] {$X$};
    \node[main] (Y) [below right=7mm and 8mm of H] {$Y$};

    \draw[arrow] (X) -- (Y);
    \draw[arrow, dashed] (H) -- node[lbl, left, yshift=4pt] {covariate shift} (X);
    \draw[arrow, densely dashdotted] (H) -- node[lbl, right, yshift=4pt] {outcome shift} (Y);

\end{tikzpicture}
\caption{Graphical model of heterogeneity regimes. Dashed arrows indicate possible dependencies between the client variable $H$, the covariates and/or outcome mechanism, depending on the regime.}
\label{fig:graphical_models}
\end{figure}

\textbf{General model.}
In full generality, the combined data admit the decomposition
\begin{equation}\label{eq:full_hetero}
    P(X,Y,H) = P(Y\mid X,H)\,P(X\mid H)\,P(H),
\end{equation}
which allows both the covariate distribution and the outcome mechanism to vary across clients. This formulation captures realistic scenarios in which institutions differ both in their underlying populations and in the relationship between features and outcomes.
The three regimes below correspond to restrictions of \eqref{eq:full_hetero}.
\begin{assumption}[Homogeneous]\label{as:iid}
    $(X,Y)\indep H.$
\end{assumption}
In this scenario, all clients are drawn from the same population. Equivalently, $P(X\mid H)=P(X)$ and $P(Y\mid X,H)=P(Y\mid X)$. Pooling data only increases the effective sample size, while $H$ carries no predictive information once $X$ is observed. Consequently, splitting on $H$ yields no (asymptotic) impurity reduction.
\begin{assumption}[Covariate shift]\label{as:covshift}
    $P(Y\mid X,H)=P(Y\mid X)$, while $P(X\mid H)$ may depend on $H$.
\end{assumption}
In this regime, clients differ in their populations but share the same outcome mechanism; in other words, there is no ``client effect''. This implies that the conditional expectation of the outcome given features, $f(x)=\mathbb E[Y\mid X=x]$, is common to all clients, and therefore that $H$ is not required for prediction once $X$ is observed. However, the marginal feature distribution is a mixture $P(X)=\sum_{k=1}^K \rho_k P(X\mid H=k)$. As a result, clients may exhibit partially or even fully non-overlapping feature supports, so that the pooled dataset covers regions of the input space that are under-represented or entirely absent in individual local datasets.
\begin{assumption}[Outcome shift]\label{as:outcome_shift}
    $P(X\mid H)=P(X)$ while $P(Y|X\new{,H})$ may depend on $H$.
\end{assumption}
In this setting, the conditional relationship of outcomes given the covariates may depend on the client. Conditioning on $X$ alone is generally insufficient: the optimal prediction rule is
$f(x,h)=\mathbb E[Y\mid X=x, H=h]$. 
Equivalently, $H$ induces an effect on $Y$ and must be included in the conditioning set. A centralized CART/RF trained on pooled data can capture this structure by treating $H$ as an input feature and splitting on it whenever it reduces impurity.

\section{Federated CART}\label{sec:fed_split_method}

In this section, we present our approach for reproducing CART's split selection in the federated setting.
At any node $\nu$, let $\mathcal S_{\nu,k}$ denote the subset of $n_{\nu,k}$ samples from client $k$ that reach $\nu$. The pooled node sample exists only as the virtual union $\mathcal S_\nu = \bigcup_{k=1}^K \mathcal S_{\nu,k}$. Our federated split selection rule has two key components: (i) constructing a finite set of candidate splits from distributed information (\Cref{subsec:candidate_splits}); and (ii) evaluating each candidate’s impurity reduction using only client-to-server summaries (\Cref{sec:split_evaluation}).

\subsection{Candidate Split Generation}\label{subsec:candidate_splits}

In centralized CART implementations \citep[e.g., \texttt{scikit-learn}][]{scikitlearn}, the pooled values $\{x_{i,j}\}_{i\in\mathcal S_\nu}$ are sorted, and all midpoints between consecutive distinct values of each feature $j$ are evaluated. While exact, this midpoint grid is impractical in the federated setting, as it would require communicating essentially all node-level feature values.

Naive alternatives are also unsatisfactory under client heterogeneity: {fixed grids} depend on feature scaling and can miss narrow high-gain regions, while {random grids} provide no guarantees and may miss low-density gaps between client supports under covariate shift.
Instead, we propose \emph{federated quantile sketching}: a {data-adaptive, communication-efficient} candidate-generation scheme that estimates the {pooled (mixture) distribution} at node $\nu$ from aggregated client summaries, and remains reliable under covariate shift, including disjoint client supports.

\textbf{Federated quantile sketching.}
We construct a compact candidate set based on pooled quantiles.
For each continuous feature $j$ and node $\nu$, client $k$ computes $B$ empirical quantiles $\mathbf q^{(j)}_{\nu,k}= (q^{(j)}_{\nu,k,1},\dots,q^{(j)}_{\nu,k,B})$ on its node data $\mathcal S_{\nu,k}$, and sends them to the server. 
The server converts $\mathbf q^{(j)}_{\nu,k}$ into a piecewise-linear approximation $\tilde F^{(j)}_{\nu,k}$ of the local empirical cumulative distribution functions (CDF) $F^{(j)}_{\nu,k}$, and then constructs a pooled CDF estimate via a mixture:
\vspace{-0.15cm}
\begin{equation}\label{eq:quantile_agg}
    \textstyle\tilde F_\nu^{(j)}(x)=\sum_{k=1}^K \frac{n_{\nu,k}}{n_\nu}\,\tilde F^{(j)}_{\nu,k}(x),
\end{equation}
with $n_\nu=\sum_{k=1}^K n_{\nu,k}$.
Finally, candidate thresholds are selected as the interior quantiles of the reconstructed pooled~CDF:
\begin{equation*}
    \mathcal T_\nu^{(j)}=\big\{\big(\tilde F_\nu^{(j)}\big)^{-1}\big(\tfrac{b}{B}\big):\ b=1,\dots,B-1\big\}.
\end{equation*}
Importantly, we do {not} average local quantiles (which in general does not yield pooled quantiles); instead we aggregate {CDF sketches} and invert the pooled CDF.
\begin{theorem}[Uniform rank error of the reconstructed CDF]\label{th:global_rank_preservation}
    Let $F_\nu^{(j)}(x)=\sum_{k=1}^K \frac{n_{\nu,k}}{n_\nu}F_{\nu,k}^{(j)}(x)$ be the pooled empirical CDF for feature $j$ at node $\nu$, and let $\tilde F_\nu^{(j)}$ be obtained by linearly interpolating between the $B$ reported quantile points per client and mixing as above. Then,
    $$\textstyle\sup_x\big|\tilde F_\nu^{(j)}(x)-F_\nu^{(j)}(x)\big|\le \frac{1}{B}.$$
\end{theorem}

\begin{corollary}[Approximation of centralized midpoint splits]\label{cor:epsilon_approx}
Let $\mathcal T_{\mathrm{cent}}$ denote the centralized midpoint candidates at node $\nu$, and let
$\mathcal T_{\mathrm{fed}}=\bigcup_{j=1}^d\{(j,t): t\in \mathcal T_\nu^{(j)}\}$
be the quantile-based grid. For any $(j,t_{\mathrm{cent}})\in\mathcal T_{\mathrm{cent}}$, there exists $(j,t_{\mathrm{fed}})\in T_{\mathrm{fed}}$ such that the induced left-child assignments disagree on at most a $\frac{3}{2B}$ fraction of pooled samples:
\[
\textstyle\frac{1}{n_\nu}\sum_{i\in\mathcal S_\nu}\mathbb{1}\left(\mathbb{1}(x_{i,j}\le t_{\mathrm{cent}})\neq \mathbb{1}(x_{i,j}\le t_{\mathrm{fed}})\right)\le \frac{3}{2B}.
\]
\end{corollary}
\Cref{th:global_rank_preservation,cor:epsilon_approx} (proved in Appendix~\ref{app:proof_quantile_sketching}) show that {$B$ directly controls the rank resolution} of the candidate grid: increasing $B$ makes the federated quantile thresholds {closer to the centralized midpoint thresholds} in terms of induced child assignments. For example, setting $B=100$ percentiles implies at most $1.5\%$ disagreement in child assignments while requiring only $B-1$ candidates per feature, independent of $n_\nu$. These results allow precise control of the trade-off between communication cost and candidate-threshold fidelity to the centralized approach.

Our quantile-based candidates 
offer several advantages to the fixed-width histograms proposed by \citet{kalloori2022cross} for three reasons: (i) they adapt to local data density, ensuring critical split points are preserved as well as possible for a given budget $B$; (ii) they are more robust against outliers; and (iii) %
they ensure an efficient use of bins by enforcing a minimum sample fraction of $1/B$ within each bin (see experiment in \Cref{subsec:outliers}).

\textbf{Top-$L$ feature shortlisting for homogeneous data.} 
Under \Cref{as:iid}, local impurity gains concentrate around the pooled gain, so local rankings tend to agree with the global ranking. This enables a communication-efficient preselection step: each client $k$ computes for each feature $j$ its best local gain
$
M_{\nu,k}(j)=\max_{t}\ \Delta I(j,t;\mathcal S_{\nu,k})
$
and sends only its top-$L$ features (with associated gains). The server merges these reports (e.g., taking the union of the local top-$L$ sets) to form a reduced feature set $\mathcal L_\nu$ 
on which it performs quantile sketching and split evaluation. A supporting guarantee for retaining the pooled-best feature under \Cref{as:iid} is given in Appendix~\ref{app:topL_screening}.  

\textbf{Discrete features.}
For small-cardinality categorical variables, clients can send exact category counts and the server evaluates all admissible splits with negligible cost. For high-cardinality categories, we use the standard ordering trick known as Fisher grouping \citep{fisher1958grouping}: we order categories by their within-node outcome mean (regression) or class proportion (binary classification), treat the ordered categories as an ordinal feature, and scan contiguous splits.

\subsection{Federated Impurity Reduction}\label{sec:split_evaluation}

{
\new

Given the candidate set $\mathcal T_\mathrm{fed}$ obtained in \Cref{subsec:candidate_splits}, the remaining task is to evaluate the impurity reduction \eqref{eq:impurity_reduction} for each candidate split $(j,t)$ without exchanging raw data. We first describe \emph{AvgImp}, a natural evaluation rule for the homogeneous setting, and then introduce our main approach for the general heterogeneous setting.

\textbf{Averaging local impurity gains under client homogeneity (\emph{AvgImp}).}
Under \Cref{as:iid}, local split objectives are expected to align with the centralized CART objective. This motivates a simple federated evaluation rule based on averaging local impurity reductions. For each candidate split $(j,t)$, client $k$ computes its local gain $\Delta I(j,t;\mathcal S_{\nu,k})$, and the server forms the weighted average
\[
\widehat{\Delta I}_{\mathrm{AvgImp}}(j,t;\mathcal S_\nu)
:=\sum_{k=1}^K\frac{n_{\nu,k}}{n_\nu}\,\Delta I(j,t;\mathcal S_{\nu,k}).
\]
The split is then selected by maximizing $\widehat{\Delta I}_{\mathrm{AvgImp}}$ over the $\mathcal T_\mathrm{fed}$. In practice, this can be combined with local Top-$L$ screening: each client reports only its best candidates according to local gains, and the server evaluates the averaged criterion on the resulting shortlist.

\begin{theorem}[\emph{AvgImp} approximation error on homogeneous data]\label{th:impurity_iid}
Under \Cref{as:iid} and $\mathbb E[Y^2]<\infty$, we have:
\[
\Big|\widehat{\Delta I}_{\mathrm{AvgImp}}(j,t;\mathcal S_\nu)-\Delta I(j,t;\mathcal S_\nu)\Big|
=O_p\Big(\frac{K}{n_\nu}\Big).
\]
\end{theorem}

While this approximation bound becomes uninformative in highly fragmented regimes where $K$ is comparable to $n$, it is meaningful in typical cross-silo federated learning scenarios, where $n\gg K$ (i.e., a small number of clients each with large datasets, such as hospitals with patient records).

However, the guarantee above relies on \Cref{as:iid}. In heterogeneous settings, we show that averaging local impurities does not generally recover the pooled impurity, as described in the following decomposition.

\begin{theorem}[General impurity decomposition]\label{th:impurity_decomposition}
The impurity reduction for a given candidate decomposes as
\begin{equation}\label{eq:impurity_decomposition}
\Delta I(j,t;\mathcal S_\nu)
= \sum_{k=1}^K \frac{n_{\nu,k}}{n_\nu} \Delta I(j,t;\mathcal S_{\nu,k}) + \Delta \mathcal E(j,t;\nu),
\end{equation}
where $\Delta \mathcal E(j,t;\nu)$ is a split-dependent heterogeneity correction, given in \Cref{app:proof_impurity_decomposition}. It corresponds to the change in the Jensen-gap term $\mathcal E(\nu)$ caused by splitting the within-node mixture of client distributions.
For standard strictly concave impurity functions (e.g., Gini index, entropy, and variance), we have $\mathcal E(\nu)\ge 0$, with 
$\mathcal E(\nu)=0$ if and only if
\[
\forall (k,k'),\quad
\hat{\mathbb E}[Y\mid \mathcal{S}_\nu,H=k]
=
\hat{\mathbb E}[Y\mid \mathcal{S}_\nu,H=k'] ,
\]
that is, the within-node empirical outcome means, or class probabilities, coincide across all clients.
\end{theorem}

Thus, when clients hold heterogeneous distributions, split decisions based on \textit{AvgImp} will deviate from those produced by a centralized CART procedure. This discrepancy motivates an alternative evaluation rule that remains exact even under heterogeneous data.

\textbf{Exact impurity evaluation via additive summaries.}
Our main evaluation rule relies on the novel observation that standard CART impurities admit additive sufficient statistics, as shown in  \Cref{thm:exact_impurity_reduction}. These statistics can be aggregated across clients to recover the centralized impurity reduction exactly for any candidate split.

\begin{theorem}[Exact centralized evaluation via additive summaries]\label{thm:exact_impurity_reduction}
For standard CART criteria, there exists a deterministic map $\Psi$ and per-client summaries $\mathbf s_{\nu,k}$ such that the pooled impurity reduction can be written as
\begin{equation}\label{eq:exact_impurity_reduction}
\Delta I(j,t;\mathcal S_\nu)
=\Psi(\mathbf s_\nu)-\Big(\tfrac{n_L}{n_\nu}\Psi(\mathbf s_L)+\tfrac{n_R}{n_\nu}\Psi(\mathbf s_R)\Big),
\end{equation}
with aggregated statistics $\mathbf s_m=\sum_{k=1}^K\mathbf s_{m,k}$ for $m\in\{\nu,L,R\}$.
\end{theorem}

In practice, for each candidate $(j,t)$, client $k$ sends only its left-child summary $\mathbf s_{L,k}$. The server then aggregates these summaries into $\mathbf s_L$, recovers $\mathbf s_R=\mathbf s_\nu-\mathbf s_L$, and evaluates \eqref{eq:exact_impurity_reduction}. This procedure exactly replicates centralized split evaluation, irrespective of covariate or outcome shift.

\emph{Examples.} For $C$-class classification, $\mathbf s_\nu=(n_{\nu,1},\dots,n_{\nu,C})$ and
\(
\Psi_{\mathrm{Gini}}(\mathbf s_\nu)
=
1-\sum_{c=1}^C(n_{\nu,c}/n_\nu)^2,
\)
so each client sends $C$ scalars per candidate. For regression, $\mathbf s_\nu=(n_\nu,S_\nu,Q_\nu)$, with $S_\nu=\sum y_i$ and $Q_\nu=\sum y_i^2$, and
\(
\Psi_{\mathrm{Var}}(\mathbf s_\nu)
=
Q_\nu/n_\nu-(S_\nu/n_\nu)^2,
\)
so each client sends three scalars per candidate.

In summary, \emph{AvgImp} provides a natural evaluation rule under homogeneous client distributions, but its guarantees do not extend to violations of \Cref{as:iid}. The additive-summary rule in \Cref{thm:exact_impurity_reduction} is therefore our primary evaluation rule when exact centralized split selection is required.
}

\textbf{Outcome shift: splitting on $H$ with no extra communication.}
When \Cref{as:outcome_shift} holds {\new (which can be tested in a federated manner, see \Cref{test:outcome_shift})}, the target rule is $f(x,h)=\mathbb E[Y\mid X=x,H=h]$, and optimal trees may require splits on $H$. Let $\mathcal K_\nu=\{k:\ n_{\nu,k}>0\}$ be the set of active clients at node $\nu$. 
Once the server has collected the per-client node summaries $\{\mathbf s_{\nu,k}\}_{k\in\mathcal K_\nu}$—already needed to evaluate feature-based splits—it can also evaluate splits on $H$ \emph{at no additional communication cost}, i.e., without requesting any extra messages from clients.
Indeed, any partition $\mathcal K_\nu=\mathcal K_L\cup\mathcal K_R$ induces
$
\textstyle\mathbf s_{\nu,L}=\sum_{k\in\mathcal K_L}\mathbf s_{\nu,k}
$ and 
$
\mathbf s_{\nu,R}=\sum_{k\in\mathcal K_R}\mathbf s_{\nu,k},
$
and the corresponding gain follows directly from \eqref{eq:exact_impurity_reduction} using the same map $\Psi$.

Although there are $2^{|\mathcal K_\nu|-1}$ partitions, the optimal split can be found efficiently: for regression and binary classification, it suffices to sort sites by their within-node means $\bar y_{\nu,k}$ and scan contiguous splits \citep{fisher1958grouping}, for a cost $O(|\mathcal K_\nu|\log|\mathcal K_\nu|)$. When $H$ has many categories, treating $H$-splits as Fisher groupings is essential; otherwise trees waste depth constructing arbitrary partitions of client IDs.

\section{Federated Forests}\label{sec:fed_rf}
We now show how to lift the federated CART primitives from \Cref{sec:fed_split_method}---candidate generation and summary-based split evaluation---into a scalable Random Forest training protocol, which we call \emph{FedForest}. The protocol, described in \Cref{alg:fedforest}, rests on three key components.

\textbf{Client-stratified bootstrap.} 
For each tree, client $k$ draws (with replacement) a bootstrap sample of size $n_k$ from its local data $\mathcal D_k$, and the tree is trained on the virtual union of these resamples across clients (Alg.~\ref{alg:fedforest}, Line 4). This preserves the pooled mixture weights $\rho_k=n_k/n$ in every tree. In contrast, a naive pooled bootstrap can under-represent—or entirely omit—small clients.

\textbf{Feature subsampling.}
As in standard Random Forests, we introduce tree-level randomness via feature subsampling to promote diversity across trees. For each active node, the server samples a subset of features $\mathcal{J}_\nu$ to evaluate (Alg.~\ref{alg:fedforest}, Line 6). In the federated setting, this also reduces communication, as fewer feature--threshold pairs need to be evaluated and transmitted by each client. It can further be combined with Top-$L$ shortlisting to achieve additional savings.

\textbf{Level-wise batched training.} 
Tree growth is sequential \emph{within} each tree but independent \emph{across} trees. We leverage this by training all trees level-wise (Alg.~\ref{alg:fedforest}, Lines 5--18): at each depth, the server batches active nodes across all trees into a single request. Clients answer for the whole batch in parallel (Alg.~\ref{alg:fedforest}, Lines 7--11 and 15--17), and the server updates all nodes independently. As a result, training $T$ trees of depth $M$ requires only $O(M)$ synchronized rounds, rather than scaling with the total number of nodes.\looseness=-1

\begin{algorithm}[t!]
  \caption{\textsc{FedForest}}\label{alg:fedforest}
  \begin{algorithmic}[1]
    \STATE {\bfseries Input:} \#trees $T$, depth $M$, client subsampling rate $r$, sketch size $B$.
    \STATE {\bfseries Output:} Forest predictor $\hat f$.
    
    \STATE \textbf{Init:} Server initializes root nodes $\mathcal A = \{(\tau, \text{root})\}_{\tau=1}^T$.
    \STATE For each $\tau$, server selects client subsets $\mathcal{K}_\tau$ (size $\lceil rK \rceil$); Clients draw bootstrap samples.

    \FOR{depth level $\ell=0$ {\bfseries to} $M-1$}
      \STATE \textbf{Server:} Broadcasts previous split decisions and selects features $\mathcal J_\nu$ for active nodes $\nu \in \mathcal A$.

      \STATE \textbf{Clients (Parallel):} For all $\tau$ and $k \in \mathcal{K}_\tau$:
      \STATE \quad Update local node membership for active nodes $\nu$.
      \STATE \quad Compute sketches $\mathbf q^{(j)}_{\nu,k}$ for continuous $X \in \mathcal J_\nu$ and local node summaries $\mathbf s_{\nu,k}$ (for $H$).
      \STATE \quad Send sketches and summaries to Server.

      \STATE \textbf{Server (Candidate Generation):}
      \STATE \quad \textbf{For $X$:} Aggregate sketches $\to$ approximate global CDF via Eq.~\ref{eq:quantile_agg} $\to$ candidates $\mathcal T_\nu^{(X)}$.
      \STATE \quad \textbf{For $H$:} Compute site means from $\mathbf s_{\nu,k}$; sort site indices by mean; generate splits $\mathcal T_\nu^{(H)}$.
      \STATE \quad Broadcast candidates $\mathcal T_\nu = \mathcal T_\nu^{(X)} \cup \mathcal T_\nu^{(H)}$.

      \STATE \textbf{Clients (Parallel):} For all $\tau$ and $k \in \mathcal{K}_\tau$:
      \STATE \quad Compute \textbf{left-child summaries} $\mathbf s_{L,k}(j,t)$ for all candidates $(j,t) \in \mathcal T_\nu$ and send to Server.

      \STATE \textbf{Server:} Aggregates $\mathbf s_L = \sum_k \mathbf s_{L,k}$; evaluates impurity gain via \Cref{eq:exact_impurity_reduction};
      \STATE \quad Selects best split $(j^*, t^*)$ and updates $\mathcal A$.
    \ENDFOR
    
    \STATE \textbf{Finalize:} Server computes leaf values from final node summaries.
  \end{algorithmic}
\end{algorithm}

\textbf{Client subsampling (Optional).} 
For a given tree, the server can randomly sample a client subset $\mathcal K'\subseteq[K]$ (Alg.~\ref{alg:fedforest}, Line 4) to reduce communication. We recommend this mainly as an acceleration knob under approximately homogeneous clients. Under covariate shift, it changes the mixture distribution and may alter split decisions; under outcome shift, it can omit informative values of $H$.

\textbf{Communication and computation costs.} We summarize and compare the communication and complexities of different methods in \Cref{tab:computation_communication_costs}, and derive the per-node scalar costs in \Cref{app:communication}.

\section{Numerical Results}\label{sec:simulations}

\textbf{Experimental setting.}
We consider a centralized \texttt{scikit-learn} Random Forest trained on pooled data as a reference, as well as two federated baselines reimplemented due to the lack of public code: Federated Histogram~\citep{kalloori2022cross}, adapted to regression via node-wise sufficient statistics for squared loss, and one-shot Local Ensembling~\citep{hauschild2022federated}, where clients train local forests and the server aggregates predictions. We also include Local Learning, where each client trains a forest on its own dataset $\mathcal{D}_k$.
We compare these with two variants of our approach: FedForest-Quantiles, which combines quantile sketches with exact split evaluation via additive sufficient statistics, and FedForest-AvgImp/Top-$L$, which uses Top-$L$ feature shortlisting followed by impurity evaluation based on averaged local gains.

All methods use $T=50$ trees, maximum depth $8$, minimum leaf size $5$, feature subsampling ($d/3$ for regression, $\sqrt d$ for classification), and client-stratified bootstrapping. FedForest and Federated Histogram use $B=32$ candidates per feature, while AvgImp/Top-$L$ uses $L=3$ screened features. For methods that split on both $X$ and $H$, $H$ is always included as a candidate (i.e., exempt from feature subsampling), whereas baselines without client splits use only $X$. Since standard regression forests do not natively handle unordered categorical splits on $H$, centralized RF uses Fisher encoding at the root node, while FedForest recomputes it at each node from shared summaries.

\textbf{Synthetic data.}
We simulate a horizontal multicenter setting with $K=10$ clients. Client $k$ observes $n_k=200$ i.i.d.\ samples $(H_i=k, X_i^{(k)}, Y_i^{(k)})$, where $X_i^{(k)} \sim \mathcal{N}(\mu_k, \alpha_k I_d)$ and $\mu_k = [(-1)^k \gamma, \dots, (-1)^k \gamma]$. Here, $\gamma \ge 0$ and $\alpha_k > 0$ control covariate heterogeneity: smaller $\gamma$ and larger $\alpha_k$ induce stronger overlap across clients. Outcomes are generated as $Y_i^{(k)} = f(X_i^{(k)}) + \delta_k + \varepsilon_i$, with $\varepsilon_i \sim \mathcal{N}(0, .5)$ and $\delta_k = (-1)^k \delta$. Thus, $\delta$ controls outcome heterogeneity, while $f$ is a tree-based regression function (see \Cref{sec:add_sims}). We set $d=20$ and study three settings: (i) \textbf{homogeneous} ($\gamma=\delta=0$, $\alpha_k=1$), (ii) \textbf{covariate shift} ($\gamma=3$, $\alpha_k=0.5$, $\delta=0$), and (iii) \textbf{outcome shift} ($\gamma=0$, $\alpha_k=1$, $\delta=1.5$).

\begin{figure*}[htbp]
    \centering
    \includegraphics[width=\textwidth]{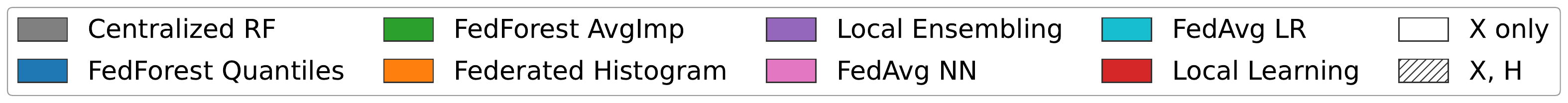}
    \label{fig:legend}
    
    \vspace{.3em} %
    
    \begin{subfigure}[b]{0.49\textwidth}
        \centering
        \includegraphics[width=\linewidth,trim={0cm 1.5cm 0cm 1cm},clip]{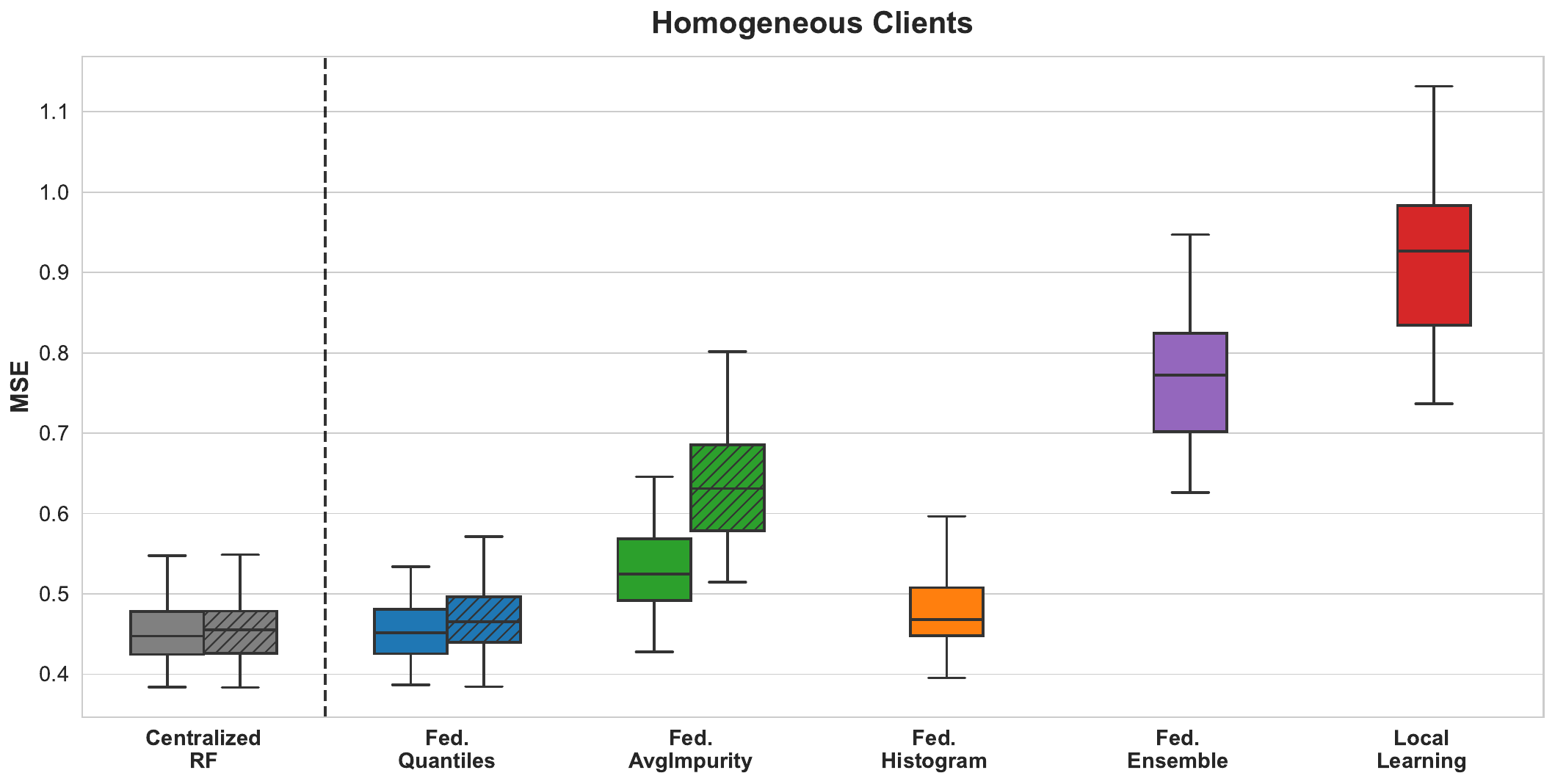}
        \caption{Homogeneous clients}
        \label{fig:homogeneous_benchmark}
    \end{subfigure}
    \hfill %
    \begin{subfigure}[b]{0.49\textwidth}
        \centering
        \includegraphics[width=\linewidth,trim={0cm 1.5cm 0cm 1cm},clip]{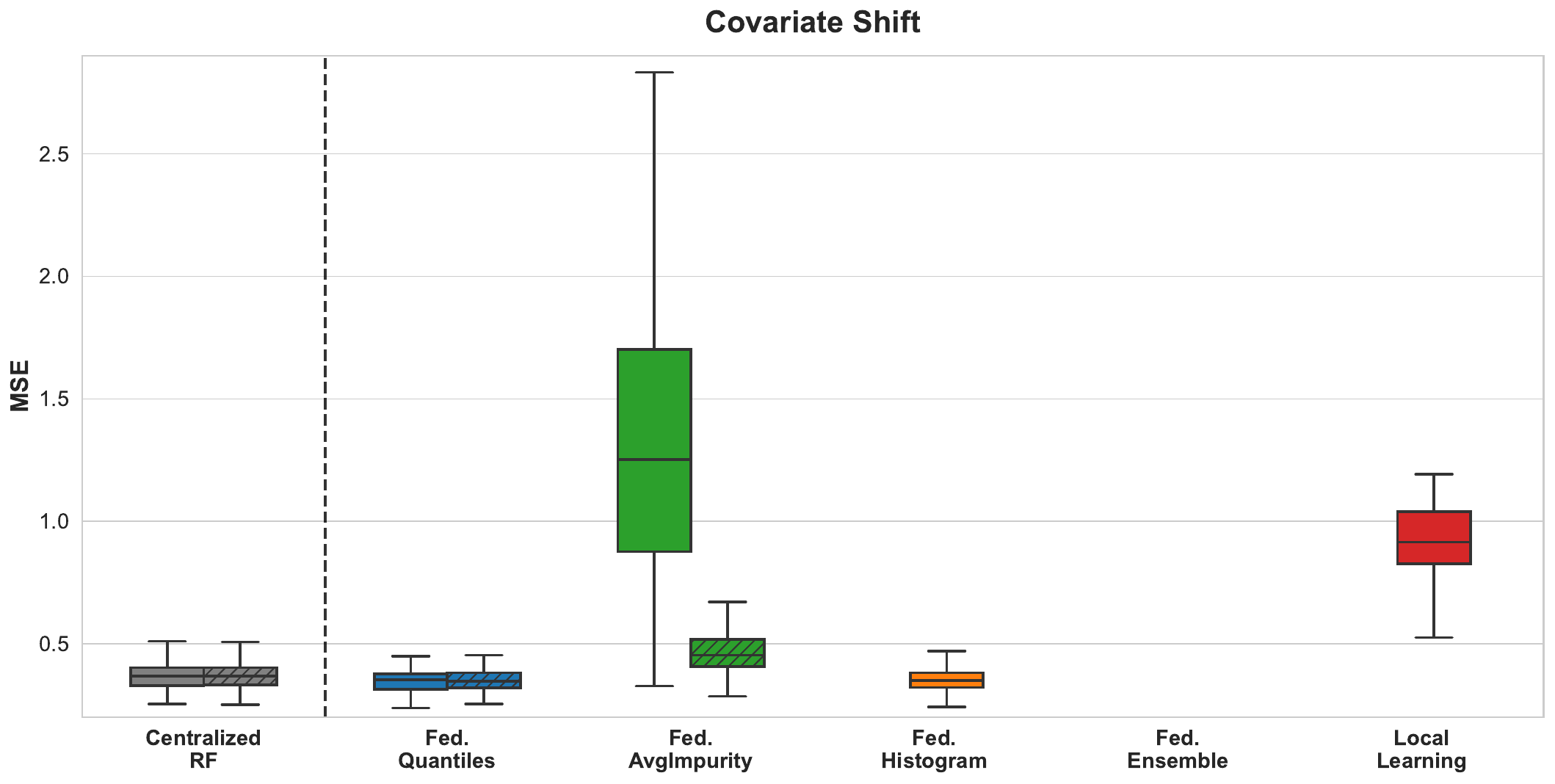}
        \caption{Covariate shift}
        \label{fig:covshift}
    \end{subfigure}

    \vspace{.5em} %

    \begin{subfigure}[b]{0.49\textwidth}
        \centering
        \includegraphics[width=\linewidth,trim={0cm 1.275cm 0cm 1cm},clip]{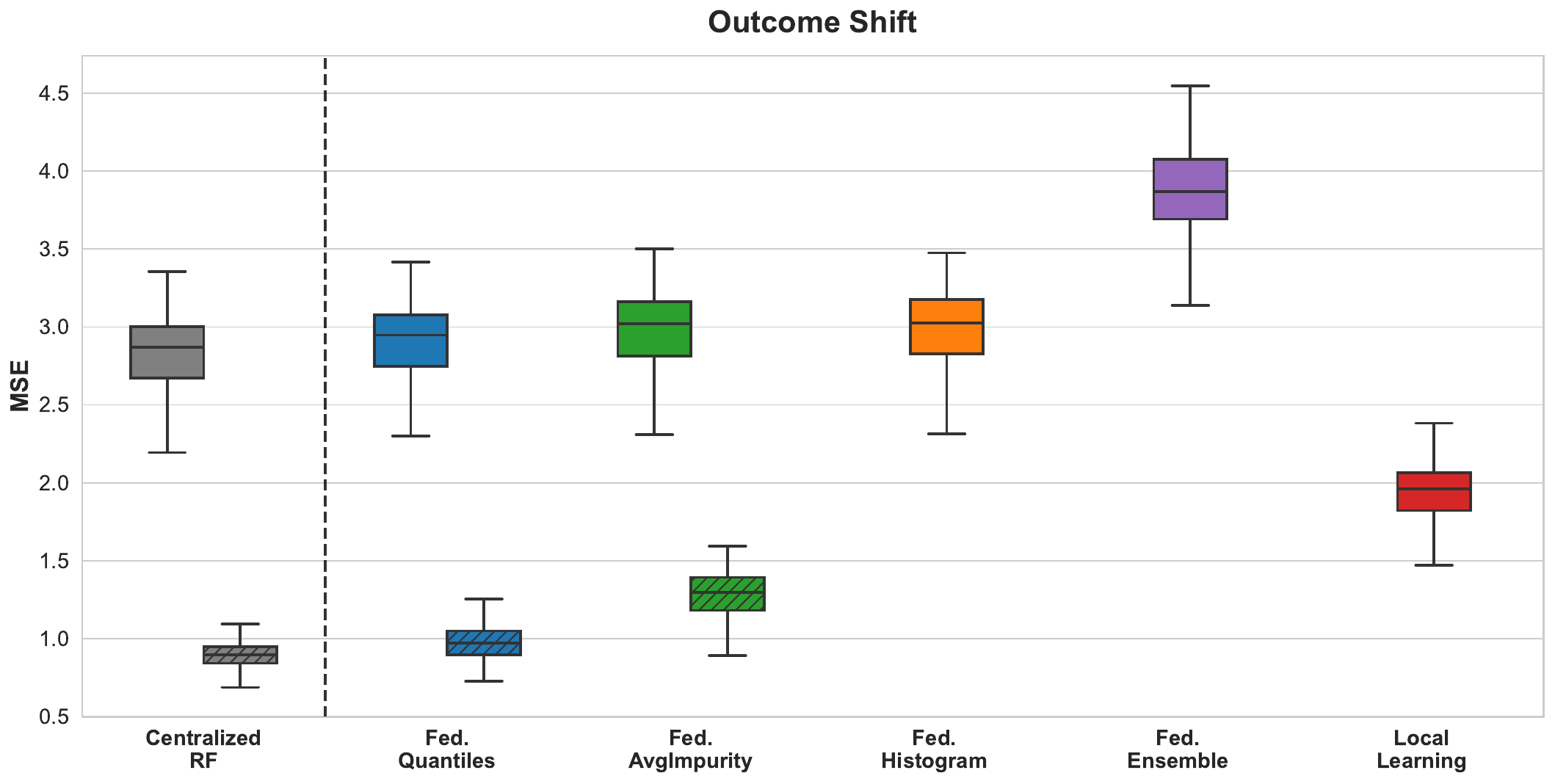}
        \caption{Outcome shift}
        \label{fig:outcome_heterogeneous_benchmark}
    \end{subfigure}
    \hfill
    \begin{subfigure}[b]{0.49\textwidth}
        \centering
        \includegraphics[width=\linewidth,trim={0cm 1.5cm .25cm .9cm},clip]{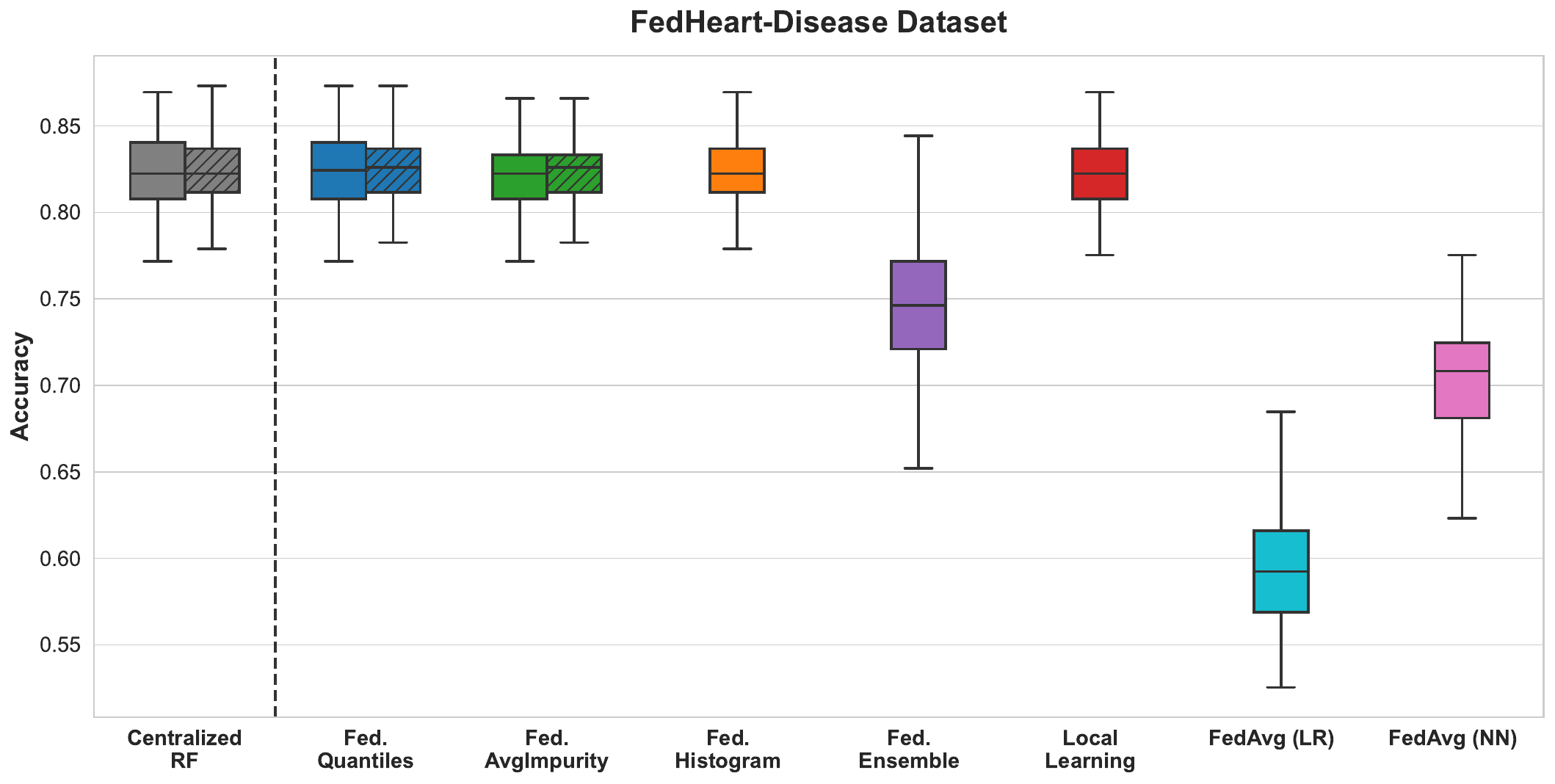}
        \caption{FedHeart-Disease dataset}
        \label{fig:fedheart_disease_benchmark}
    \end{subfigure}

    \caption{Performance comparison on simulated data and the real-world FedHeart dataset.}
    \label{fig:all_benchmarks}
\end{figure*}

Figure~\ref{fig:homogeneous_benchmark} reports test MSE on a $30\%$ held-out mixture sample from the client mixture in the homogeneous setting (averaged over 150 draws). FedForest-Quantiles (with or without $H$) and Federated Histogram (which does not split on $H$) perform on par with centralized RF, with FedForest-Quantiles showing lower variance. Local learning and local ensembling perform worse due to limited per-client sample sizes; although unbiased here, they underutilize data relative to federated and centralized methods. FedForest-AvgImp is slightly weaker and degrades when including $H$, but approaches centralized performance with larger $B$ or higher $L$.

In the covariate-shift setting (\Cref{fig:covshift}), client supports are nearly disjoint, and the target function requires a split between them. This stresses methods relying on local impurity estimates. FedForest-Quantiles and Federated Histogram remain close to centralized RF since they evaluate splits using pooled statistics. In contrast, FedForest-AvgImp fails in this regime: averaging client-wise gains can miss global splits when each client observes only one side of the partition. Allowing AvgImp to split on $H$ mitigates this by acting as a proxy for client identity. Local ensembling is omitted due to extremely high error ($\approx 141.2$), as models trained per client extrapolate poorly across disjoint supports.\looseness-1

Finally, Figure~\ref{fig:outcome_heterogeneous_benchmark} shows the outcome-shift setting. Allowing splits on $H$ significantly improves centralized RF and FedForest variants, which can adapt to client-specific offsets, whereas baselines restricted to $X$ cannot capture this heterogeneity and perform substantially worse. Federated Histogram, not designed to split on $H$, is particularly affected in this regime.

\textbf{Real data.}
We further validate our approach on the FedHeart Disease dataset from the FLamby federated benchmark \citet{ogier2022flamby}. This binary classification task comprises $740$ patients and $13$ clinical covariates distributed across four institutions. %
We compare our approach against two standard federated baselines evaluated by \citet{ogier2022flamby}: \textbf{FedAvg LR}, a logistic regression model trained via Federated Averaging \citep{mcmahan2017communication}, and \textbf{FedAvg NN}, a two-layer neural network trained using the same protocol.
As shown in \Cref{fig:fedheart_disease_benchmark}, both FedForest Quantiles and FedForest AvgImp achieve strong predictive performance, on par with the tree-based baselines (Federated Histograms and Local Learning). In contrast, the parametric FedAvg baselines (logistic regression and neural network) perform substantially worse. These results illustrate how flexible tree-based methods are often better suited for heterogeneous tabular datasets.\looseness=-1

{\new We further validate our findings on 13 public datasets from OpenML and scikit-learn in \Cref{app:benchmark}.}

{\new
\section{Conclusion}
We propose \emph{FedForest}, a federated forest training procedure with formal guarantees, showing that centralized CART and Random Forest behavior can be faithfully recovered from decentralized data without sharing individual samples. By reconstructing global split criteria from local summaries, our approach provides a principled foundation for extending the broader family of CART-based methods---including causal, quantile, and survival forests---to heterogeneous federated settings.

While exchanging only summary statistics reduces direct privacy exposure, these aggregates may still leak sensitive information about local data. Our protocol can be further strengthened with additional safeguards, including structural regularization (e.g., depth constraints and minimum leaf sizes), adaptive binning ($B_\nu < \min_k n_{\nu,k}$), secure aggregation, and client abstention for small local partitions. More broadly, an interesting direction for future work is to extend our framework with formal differential privacy guarantees and to study the resulting privacy--utility trade-offs.}

\newpage

\bibliography{bib}
\bibliographystyle{unsrtnat}

\begin{appendices}
    \include{appendix} 
\end{appendices}

\end{document}

%% file: appendix.tex
\appendix
\onecolumn

\section{FedForest Communication and Computation Costs}\label{app:communication}

In this section, we detail the communication and computation complexities of our proposed \textsf{FedForest} protocol and compare them against standard baselines. Let $T$ be the number of trees, $M$ the maximum tree depth, $d$ the number of features, $K$ the number of clients, $n_k$ the number of samples at client $k$, and $n$ the total number of samples. 

\paragraph{Communication per node.}
We count communication in transmitted scalars for one node $\nu$ and one client $k$, with $B$ representing the quantile levels (sketch size), and $S$ representing the sufficient-statistic size (e.g., $S=3$ for regression and $S=C$ for $C$-class classification). All costs below are strictly independent of the node sample sizes $(n_{\nu,k},n_\nu)$.

\begin{itemize}[leftmargin=*]
    \item \emph{FedForest (exact gain):} Communication occurs in two steps. (i) \textbf{Sketching:} the client sends $d$ univariate quantile sketches of size $B$, for a cost of $dB$. (ii) \textbf{Evaluation:} for each tested threshold in $\mathcal T$, the client returns the left-child sufficient statistics, costing $S$ scalars per candidate. The total cost of FedForest per node is therefore $O(dBS)$.
    \item \emph{Top-$L$ shortlisting:} If the server evaluates only $|\mathcal L_\nu| = L \ll d$ shortlisted features, the evaluation term strictly reduces to $O(LBS)$.
    \item \emph{AvgImp (Top-$L$ + averaged local gains):} If clients compute their impurity gains locally at the candidate thresholds of the shortlisted features, they send \emph{one scalar per threshold} instead of $S$. Combining Top-$L$ with AvgImp therefore reduces the evaluation communication to $O(LB)$, at the price of shifting computation from the server to the clients for the local gain evaluation.
\end{itemize}

\paragraph{Computation complexity.}
In a centralized Random Forest, all computation is localized at the server, requiring $O(T 2^M d n)$ operations to route data and evaluate splits across all nodes. In \textit{FedForest}, this burden is decoupled:
\begin{itemize}[leftmargin=*]
    \item \emph{Client-side computation:} At each level of the tree, a client must route its local data and update the quantile sketches for the active nodes. Over $T$ trees and $M$ depth levels, this requires $O(T M d n_k)$ operations, scaling only with the client's local dataset size. For the \textsf{AvgImp} variant, clients additionally sort feature importances, adding an $O(T 2^M d \log d)$ overhead.
    \item \emph{Server-side computation:} The server never touches raw data. It solely aggregates $K$ sketches of size $B$ across $d$ features for up to $2^M$ nodes per tree. This results in a server computation complexity of $O(T 2^M d B K)$.
\end{itemize}

\begin{table}[htbp]
    \centering
    \caption{Communication and computation costs required for the entire training phase across centralized, local, and federated methods.}
    \label{tab:computation_communication_costs}
    \resizebox{\textwidth}{!}{%
    \begin{tabular}{l||ll|ll}
        \toprule
        \textbf{Method} & \textbf{Comm.: Client $\to$ Server} & \textbf{Comm.: Server $\to$ all Clients} & \textbf{Computation: Client} & \textbf{Computation: Server} \\
        \midrule
        \textbf{Centralized RF} & N/A & N/A & N/A & $O(T2^Mdn)$ \\
        \addlinespace
        \textbf{FedForest Quantiles} & $O(T2^MdB)$ & $O(KT2^MdB)$ & $O(TMdn_k)$ & $O(T2^MdBK)$ \\
        \addlinespace
        \textbf{FedForest AvgImp Top-$L$} & $O(T2^M(d + LB))$ & $O(KT2^MLB)$ & $O(TMdn_k + T2^Md\log d)$ & $O(KT2^M(d + LB))$ \\
        \addlinespace
        \textbf{Federated Histograms} & $O(T2^MdB_{\text{hist}})$ & $O(KT2^M)$ & $O(TMdn_k)$ & $O(T2^MdB_{\text{hist}}K)$ \\
        \addlinespace
        \textbf{Local Ensembling} & $O(T2^M)$ & N/A & $O(TMdn_k)$ & N/A \\
        \addlinespace
        \textbf{Parametric (FedAvg)} & $O(T_{\text{comm}}P)$ & $O(KT_{\text{comm}}P)$ & $O(T_{\text{comm}}Pn_k)$ & $O(T_{\text{comm}}KP)$ \\
        \bottomrule
    \end{tabular}%
    }
\end{table}

Table~\ref{tab:computation_communication_costs} highlights several scalability properties of federated tree training. First, the communication cost of \textit{FedForest} is independent of the total sample size $n$: once local summaries have been computed, communication depends only on the forest structure $(T,M)$, the feature dimension $d$, and the sketch resolution $B$. Second, Top-$L$ shortlisting reduces the number of feature--threshold pairs exchanged and evaluated. In high-dimensional regimes with $L\ll d$, this substantially lowers the server-to-client broadcast cost, from $O(KT2^MdB)$ to $O(KT2^MLB)$ for the evaluation stage. Third, the computational load is naturally distributed: clients perform the data-dependent operations on their local samples, while the server only aggregates sketches and sufficient statistics. Unlike \textit{FedAvg}, which requires an undetermined number of communication rounds ($T_{\text{comm}}$) to converge, \textit{FedForest} proceeds level by level and therefore requires a deterministic $O(M)$ synchronized rounds for trees of depth at most $M$.

\section{Benchmark on Public Datasets}\label{app:benchmark}

We further assess robustness on 13 public tabular datasets, comprising 6 regression and 7 classification tasks from OpenML and \texttt{scikit-learn}. Since these datasets are not naturally multicentric, we simulate a federated setting by partitioning each dataset across $K=20$ clients using a Dirichlet allocation. The concentration parameter $\alpha$ controls the degree of heterogeneity: small values, such as $\alpha=0.1$, induce highly non-\emph{i.i.d.} partitions with pronounced covariate shift, while larger values yield increasingly homogeneous client distributions.

We compare FedForest against the tree-based baselines considered in the main text, as well as a federated neural-network baseline with two hidden layers of sizes $128$ and $64$.

\begin{table}[htbp]
    \centering
    \caption{Performance comparison of centralized, local, and federated learning methods across 14 datasets at varying levels of covariate heterogeneity ($\alpha$). The seven first datasets are regression tasks and performance is evaluated in terms of $R^2$. The seven next datasets are classification tasks and evaluation is in terms of balanced accuracy. Top-1 performers are written in bold, top-2 are underlined.}
    \label{tab:comprehensive_results}
    \resizebox{\textwidth}{!}{%
    \begin{tabular}{l r c c c c c c c}
        \toprule
        \textbf{Dataset} & $\alpha$ & \textbf{Cent. Scikit} & \textbf{Cent. 2-Layer NN} & \textbf{FedForest Quantiles} & \textbf{FedForest AvgImp} & \textbf{Fed. Histogram} & \textbf{Fed. Ensemble} & \textbf{Local Learning} \\
        \midrule
        \textbf{diabetes} & 10 & \underline{$0.40 \pm 0.06$} & $0.36 \pm 0.06$ & \bm{$0.41 \pm 0.07$} & \bm{$0.41 \pm 0.05$} & \bm{$0.41 \pm 0.05$} & $0.33 \pm 0.05$ &$0.26 \pm 0.09$ \\
        diabetes & 1 & \bm{$0.46 \pm 0.05$} & $0.3 \pm 0.3$ & \bm{$0.46 \pm 0.05$} & \underline{$0.44 \pm 0.06$} & \underline{$0.44 \pm 0.06$} & $0.31 \pm 0.04$ &$0.24 \pm 0.08$ \\
        diabetes & 0.1 & \underline{$0.43 \pm 0.02$} & $0.3 \pm 0.2$ & \underline{$0.43 \pm 0.04$} & \bm{$0.44 \pm 0.04$} & \underline{$0.43 \pm 0.05$} & $0.21 \pm 0.04$ &$0.25 \pm 0.02$ \\
        \midrule
        \textbf{bodyfat} & 10 & \underline{$0.92 \pm 0.02$} & $0.89 \pm 0.06$ & \bm{$0.93 \pm 0.02$} & \underline{$0.92 \pm 0.02$} & \underline{$0.92 \pm 0.02$} & $0.25 \pm 0.03$ &$0.46 \pm 0.03$ \\
        bodyfat & 1 & \underline{$0.90 \pm 0.02$} & \bm{$0.91 \pm 0.03$} & \bm{$0.91 \pm 0.02$} & \underline{$0.90 \pm 0.01$} & \underline{$0.90 \pm 0.02$} & $0.26 \pm 0.07$ &$0.3 \pm 0.2$ \\
        bodyfat & 0.1 & \underline{$0.90 \pm 0.02$} & $0.87 \pm 0.08$ & \bm{$0.91 \pm 0.02$} & $0.89 \pm 0.02$ & \bm{$0.91 \pm 0.03$} & $0.25 \pm 0.06$ &$0.6 \pm 0.1$ \\
        \midrule
        \textbf{space-ga} & 10 & $0.57 \pm 0.04$ & \bm{$0.71 \pm 0.03$} & $0.58 \pm 0.04$ & $0.57 \pm 0.04$ & $0.59 \pm 0.03$ & $0.30 \pm 0.02$ &$0.28 \pm 0.02$ \\
        space-ga & 1 & $0.57 \pm 0.04$ & \bm{$0.71 \pm 0.03$} & $0.58 \pm 0.05$ & $0.56 \pm 0.04$ & $0.58 \pm 0.04$ & $0.28 \pm 0.03$ &$0.29 \pm 0.02$ \\
        space-ga & 0.1 & $0.56 \pm 0.03$ & \bm{$0.71 \pm 0.02$} & $0.57 \pm 0.03$ & $0.56 \pm 0.02$ & $0.57 \pm 0.03$ & $0.19 \pm 0.03$ &$0.34 \pm 0.02$ \\
        \midrule
        \textbf{puma8NH} & 10 & $0.57 \pm 0.01$ & \bm{$0.67 \pm 0.01$} & $0.58 \pm 0.01$ & $0.548 \pm 0.007$ & $0.569 \pm 0.010$ & $0.194 \pm 0.003$ &$0.190 \pm 0.004$ \\
        puma8NH & 1 & $0.578 \pm 0.007$ & \bm{$0.678 \pm 0.005$} & $0.589 \pm 0.005$ & $0.563 \pm 0.006$ & $0.578 \pm 0.006$ & $0.196 \pm 0.005$ &$0.206 \pm 0.008$ \\
        puma8NH & 0.1 & $0.568 \pm 0.009$ & \bm{$0.675 \pm 0.008$} & $0.581 \pm 0.008$ & $0.556 \pm 0.008$ & $0.575 \pm 0.005$ & $0.155 \pm 0.008$ &$0.26 \pm 0.02$ \\
        \midrule
        \textbf{cpu-act} & 10 & $0.971 \pm 0.003$ & \bm{$0.977 \pm 0.002$} & $0.962 \pm 0.003$ & $0.967 \pm 0.004$ & \underline{$0.975 \pm 0.003$} & $0.29 \pm 0.02$ &$0.28 \pm 0.02$ \\
        cpu-act & 1 & \underline{$0.972 \pm 0.003$} & \bm{$0.976 \pm 0.003$} & $0.956 \pm 0.003$ & $0.966 \pm 0.003$ & \bm{$0.976 \pm 0.002$} & $0.29 \pm 0.01$ &$0.33 \pm 0.04$ \\
        cpu-act & 0.1 & $0.973 \pm 0.002$ & \bm{$0.978 \pm 0.002$} & $0.959 \pm 0.001$ & $0.970 \pm 0.002$ & \underline{$0.976 \pm 0.002$} & $0.25 \pm 0.02$ &$0.47 \pm 0.04$ \\
        \midrule
        \textbf{kin8nm} & 10 & $0.496 \pm 0.007$ & \bm{$0.905 \pm 0.005$} & $0.534 \pm 0.002$ & $0.501 \pm 0.005$ & $0.527 \pm 0.007$ & $0.252 \pm 0.003$ &$0.245 \pm 0.005$ \\
        kin8nm & 1 & $0.504 \pm 0.006$ & \bm{$0.908 \pm 0.005$} & $0.538 \pm 0.006$ & $0.506 \pm 0.006$ & $0.540 \pm 0.008$ & $0.246 \pm 0.006$ &$0.253 \pm 0.005$ \\
        kin8nm & 0.1 & $0.503 \pm 0.007$ & \bm{$0.904 \pm 0.007$} & $0.535 \pm 0.007$ & $0.520 \pm 0.008$ & $0.536 \pm 0.005$ & $0.20 \pm 0.02$ &$0.286 \pm 0.009$ \\
        \midrule
        \textbf{breast-cancer} & 10 & \underline{$0.951 \pm 0.009$} & $0.92 \pm 0.03$ & \bm{$0.96 \pm 0.02$} & \bm{$0.96 \pm 0.01$} & $0.95 \pm 0.02$ & $0.94 \pm 0.01$ &$0.90 \pm 0.03$ \\
        breast-cancer & 1 & \bm{$0.95 \pm 0.02$} & \underline{$0.92 \pm 0.05$} & \bm{$0.95 \pm 0.02$} & \bm{$0.95 \pm 0.02$} & \bm{$0.95 \pm 0.01$} & $0.91 \pm 0.02$ &$0.90 \pm 0.01$ \\
        breast-cancer & 0.1 & \bm{$0.964 \pm 0.006$} & $0.93 \pm 0.03$ & \underline{$0.960 \pm 0.007$} & \underline{$0.96 \pm 0.01$} & \underline{$0.96 \pm 0.01$} & $0.90 \pm 0.02$ &$0.91 \pm 0.02$ \\
        \midrule
        \textbf{Spectf} & 10 & $0.81 \pm 0.07$ & $0.61 \pm 0.02$ & \bm{$0.84 \pm 0.08$} & $0.79 \pm 0.09$ & \underline{$0.83 \pm 0.08$} & $0.5000$ &$0.61 \pm 0.06$ \\
        Spectf & 1 & \bm{$0.85 \pm 0.03$} & $0.63 \pm 0.07$ & $0.82 \pm 0.04$ & $0.82 \pm 0.03$ & \underline{$0.83 \pm 0.02$} & $0.5000$ &$0.60 \pm 0.03$ \\
        Spectf & 0.1 & \underline{$0.81 \pm 0.05$} & $0.61 \pm 0.06$ & \bm{$0.82 \pm 0.07$} & $0.78 \pm 0.05$ & $0.76 \pm 0.04$ & $0.5000$ &$0.63 \pm 0.06$ \\
        \midrule
        \textbf{Bioresponse} & 10 & $0.760 \pm 0.009$ & \bm{$0.771 \pm 0.005$} & \underline{$0.77 \pm 0.01$} & $0.753 \pm 0.009$ & $0.76 \pm 0.01$ & $0.63 \pm 0.02$ &$0.63 \pm 0.01$ \\
        Bioresponse & 1 & $0.753 \pm 0.009$ & \bm{$0.77 \pm 0.01$} & \underline{$0.76 \pm 0.01$} & $0.75 \pm 0.01$ & $0.754 \pm 0.008$ & $0.638 \pm 0.008$ &$0.631 \pm 0.010$ \\
        Bioresponse & 0.1 & $0.76 \pm 0.01$ & \bm{$0.77 \pm 0.02$} & \underline{$0.765 \pm 0.006$} & $0.75 \pm 0.01$ & $0.763 \pm 0.005$ & $0.62 \pm 0.03$ &$0.647 \pm 0.008$ \\
        \midrule
        \textbf{madelon} & 10 & \bm{$0.68 \pm 0.02$} & $0.55 \pm 0.01$ & \underline{$0.679 \pm 0.007$} & $0.64 \pm 0.02$ & $0.67 \pm 0.02$ & $0.59 \pm 0.01$ &$0.56 \pm 0.02$ \\
        madelon & 1 & \underline{$0.68 \pm 0.02$} & $0.54 \pm 0.01$ & \bm{$0.69 \pm 0.01$} & $0.63 \pm 0.01$ & $0.66 \pm 0.02$ & $0.59 \pm 0.03$ &$0.569 \pm 0.008$ \\
        madelon & 0.1 & \underline{$0.67 \pm 0.02$} & $0.560 \pm 0.009$ & \bm{$0.684 \pm 0.010$} & $0.66 \pm 0.02$ & $0.66 \pm 0.01$ & $0.58 \pm 0.05$ &$0.61 \pm 0.02$ \\
        \midrule
        \textbf{Wine} & 10 & \underline{$0.98 \pm 0.02$} & $0.8 \pm 0.1$ & \bm{$0.990 \pm 0.009$} & \underline{$0.98 \pm 0.02$} & \underline{$0.98 \pm 0.02$} & $0.3333$ &$0.86 \pm 0.06$ \\
        Wine & 1 & \underline{$0.97 \pm 0.02$} & $0.81 \pm 0.10$ & \bm{$0.99 \pm 0.01$} & \underline{$0.97 \pm 0.02$} & \underline{$0.97 \pm 0.02$} & $0.3333$ &$0.75 \pm 0.06$ \\
        Wine & 0.1 & \underline{$0.98 \pm 0.02$} & $0.8 \pm 0.1$ & \underline{$0.98 \pm 0.02$} & \underline{$0.980 \pm 0.007$} & \bm{$0.990 \pm 0.009$} & $0.3333$ &$0.81 \pm 0.04$ \\
        \midrule
        \textbf{Bank-Note} & 10 & \underline{$0.991 \pm 0.007$} & $0.97 \pm 0.02$ & \bm{$0.992 \pm 0.007$} & $0.989 \pm 0.006$ & \bm{$0.992 \pm 0.006$} & $0.91 \pm 0.02$ &$0.891 \pm 0.006$ \\
        Bank-Note & 1 & \underline{$0.988 \pm 0.006$} & $0.97 \pm 0.01$ & \bm{$0.989 \pm 0.004$} & \bm{$0.989 \pm 0.003$} & $0.987 \pm 0.003$ & $0.89 \pm 0.02$ &$0.88 \pm 0.02$ \\
        Bank-Note & 0.1 & $0.990 \pm 0.004$ & $0.980 \pm 0.007$ & \bm{$0.992 \pm 0.003$} & $0.990 \pm 0.004$ & $0.989 \pm 0.003$ & $0.92 \pm 0.02$ &$0.92 \pm 0.02$ \\
        \midrule
        \textbf{Wisconsin} & 10 & $0.959 \pm 0.007$ & $0.93 \pm 0.04$ & \bm{$0.968 \pm 0.008$} & \underline{$0.96 \pm 0.01$} & \bm{$0.968 \pm 0.007$} & $0.937 \pm 0.009$ &$0.93 \pm 0.02$ \\
        Wisconsin & 1 & $0.96 \pm 0.01$ & $0.95 \pm 0.02$ & \bm{$0.972 \pm 0.006$} & \bm{$0.972 \pm 0.006$} & \underline{$0.964 \pm 0.008$} & $0.930 \pm 0.008$ &$0.91 \pm 0.02$ \\
        Wisconsin & 0.1 & \underline{$0.96 \pm 0.01$} & \underline{$0.96 \pm 0.02$} & \bm{$0.97 \pm 0.01$} & \bm{$0.97 \pm 0.01$} & \bm{$0.97 \pm 0.01$} & $0.87 \pm 0.05$ &$0.93 \pm 0.03$ \\
        \bottomrule
    \end{tabular}%
    }
\end{table}

As $\alpha$ decreases, client distributions become more fragmented and local methods degrade accordingly. In contrast, \textit{FedForest Quantiles} and \textit{FedForest AvgImp} remain stable under covariate shift and closely track the centralized scikit-learn Random Forest baseline. This supports our claim that aggregating split information across clients is substantially more effective than training isolated local trees, even when those local predictors are later ensembled.

Across the benchmark, the proposed FedForest variants consistently rank among the best-performing methods. \textit{FedForest Quantiles} is typically a top performer across tasks and heterogeneity regimes, while \textit{FedForest AvgImp} achieves comparable performance in many cases. This suggests that averaging local impurity gains is an effective practical approximation, despite its stronger theoretical guarantees being restricted to homogeneous settings.

Tree-based federated methods also compare favorably with the neural-network baseline, especially in more heterogeneous regimes where flexible nonparametric splits appear better adapted to the data structure. Federated Histograms remain competitive, but \textit{FedForest Quantiles} often matches or improves upon them, supporting the use of quantile-based candidates rather than fixed-width bins.

\paragraph{Performance in Highly Fragmented Settings}
\label{app:fragmentation}

To further evaluate the robustness of our approaches in highly fragmented scenarios, we conducted additional experiments on the Banknote Authentication dataset. We varied the number of clients from $K=1$ (centralized) up to $K=100$, maintaining the same base experimental setup.

As shown in Table~\ref{tab:fragmentation}, \textbf{FedForest remains highly stable as fragmentation increases}. The performance of both FedForest variants stays remarkably close to the Centralized Scikit RF baseline across all values of $K$. In stark contrast, Local Learning and Federated Ensemble approaches degrade sharply as the data becomes divided into smaller, isolated client pools. 

These results demonstrate that the primary difficulty in fragmented settings is not the quantile construction itself, but rather the inherent failure of purely local approaches to recover the global split structure when individual clients lack sufficient data.

\begin{table}[h]
\centering
\caption{Predictive performance on the Banknote Authentication dataset across varying degrees of client fragmentation ($K=1$ to $K=100$). Values indicate mean $\pm$ standard deviation.}
\label{tab:fragmentation}
\resizebox{\textwidth}{!}{%
\begin{tabular}{l ccccccc}
\toprule
\textbf{Method} & \textbf{K=1} & \textbf{K=2} & \textbf{K=5} & \textbf{K=10} & \textbf{K=20} & \textbf{K=50} & \textbf{K=100} \\
\midrule
Centralized Scikit RF            & $0.985 \pm 0.007$ & $0.991 \pm 0.002$ & $0.986 \pm 0.009$ & $0.987 \pm 0.005$ & $0.991 \pm 0.006$    & $0.990 \pm 0.004$ & $0.991 \pm 0.001$ \\
Fed. Ensemble (Hauschild et al.) & $0.988 \pm 0.005$ & $0.986 \pm 0.005$ & $0.96 \pm 0.01$   & $0.93 \pm 0.01$   & $0.91 \pm 0.01$      & $0.87 \pm 0.02$   & $0.5000$          \\
Fed. Histogram (Kalloori et al.) & $0.996 \pm 0.001$ & $0.998 \pm 0.004$ & $0.996 \pm 0.003$ & $0.994 \pm 0.004$ & $0.994 \pm 0.005$    & $0.992 \pm 0.007$ & $0.997 \pm 0.001$ \\
FedForest AvgImp (ours)          & $0.994 \pm 0.002$ & $0.996 \pm 0.004$ & $0.993 \pm 0.006$ & $0.991 \pm 0.004$ & $0.996 \pm 0.001$    & $0.988 \pm 0.003$ & $0.989 \pm 0.003$ \\
FedForest Quantiles (ours)       & $0.988 \pm 0.008$ & $0.993 \pm 0.001$ & $0.986 \pm 0.007$ & $0.990 \pm 0.005$ & $0.9953 \pm 0.0006$  & $0.991 \pm 0.007$ & $0.993 \pm 0.002$ \\
Local Learning                   & $0.988 \pm 0.005$ & $0.980 \pm 0.005$ & $0.949 \pm 0.005$ & $0.913 \pm 0.007$ & $0.89 \pm 0.01$      & $0.79 \pm 0.03$   & $0.55 \pm 0.05$   \\
\bottomrule
\end{tabular}%
}
\end{table}

\newpage
\section{Proofs}\label{app:proofs}
\subsection{Candidate thresholds with quantile sketching}\label{app:proof_quantile_sketching}
\begin{proof}[Proof of \Cref{th:global_rank_preservation}]
    The proof proceeds in first establishing the error bound for local approximations, and then extending it to the global aggregate.
    
    \textbf{Step 1: Local Error Bound.} Consider a client $k$ transmitting $B$ exact quantiles $q_0, \dots, q_B$ corresponding to ranks $0, 1/B, \dots, 1$ of a given feature (we omit the feature index $j$ in the notation) at node $\nu$. $\forall x \in [q_i, q_{i+1}]$, the monotonicity of the CDF implies that the true rank $F_{\nu,k}(x)\in [i/B, (i+1)/B]$. Similarly, the linear interpolant $\hat{F}_{\nu,k}(x)$ constructed by the server is also bounded strictly within $[i/B, (i+1)/B]$. Since both the true function and the approximation are in the same interval of height $1/B$, their absolute difference is bounded by that height:
    $$\left| \hat{F}_{\nu,k}(x) - F_{\nu,k}(x) \right| \le \frac{i+1}{B} - \frac{i}{B} = \frac{1}{B}.$$
    
    \textbf{Step 2: Global Aggregation.} We now evaluate the error of the global reconstructed CDF $\tilde{F}_\nu(x) = \sum w_k \hat{F}_{\nu,k}(x)$ relative to the true global CDF $F_\nu(x) = \sum w_k F_{\nu,k}(x)$, where $w_k = n_{\nu,k}/n_\nu$:
    $$
    \begin{aligned}
        |\tilde{F}_\nu(x) - F_\nu(x) | &= \left| \sum_{k=1}^K w_k \left( \hat{F}_{\nu,k}(x) - F_{\nu,k}(x) \right) \right| \\
        &\le \sum_{k=1}^K w_k \left| \hat{F}_{\nu,k}(x) - F_{\nu,k}(x) \right| \quad \text{(Triangle Inequality)} \\
        &\le \sum_{k=1}^K w_k \left( \frac{1}{B} \right)\\
        &= \frac{1}{B} \sum_{k=1}^K w_k = \frac{1}{B}.
    \end{aligned}
    $$
    Thus, the global rank error is strictly bounded by $\epsilon = 1/B$, independent of the heterogeneity of the local distributions $F_{\nu,k}$.\\
    \qedhere
\end{proof}

\begin{proof}[Proof of \Cref{cor:epsilon_approx}]
We consider a centralized candidate $(j,t_{\mathrm{cent}})\in\mathcal T_{\mathrm{cent}}$. We write $F_\nu$ and $\tilde F_\nu$ for the pooled empirical CDF and its reconstruction (as in \Cref{th:global_rank_preservation}). By construction, the federated candidate set for this feature is
\[
\mathcal T_\nu^{(j)}=\Big\{\tilde F_\nu^{-1}\!\big(\tfrac{b}{B}\big):\ b=1,\dots,B-1\Big\}.
\]
Hence there exists $t_{\mathrm{fed}}\in \mathcal T_\nu^{(j)}$ such that
\[
\big|\tilde F_\nu(t_{\mathrm{fed}})-F_\nu(t_{\mathrm{cent}}\big|\le \frac{1}{2B},
\]
since the grid points $\{b/B\}$ are spaced by $1/B$. We now bound the induced disagreement in left-child assignments:
\[
\frac{1}{n_\nu}\sum_{i\in\mathcal S_\nu}\mathbb {1}\!\left(\mathbb{1}(x_{i}\le t_{\mathrm{cent}})\neq \mathbb 1(x_{i}\le t_{\mathrm{fed}})\right)
= \big|F_\nu(t_{\mathrm{fed}})-F_\nu(t_{\mathrm{cent}})\big|.
\]
By the triangle inequality and \Cref{th:global_rank_preservation},
\[
\big|F_\nu(t_{\mathrm{fed}})-F_\nu(t_{\mathrm{cent}}\big|
\le \big|F_\nu(t_{\mathrm{fed}})-\tilde F_\nu(t_{\mathrm{fed}})\big|
     +\big|\tilde F_\nu(t_{\mathrm{fed}})-F_\nu(t_{\mathrm{cent}}\big|
\le \frac{1}{B}+\frac{1}{2B}
= \frac{3}{2B}.
\]
\end{proof}

\subsection{AvgImp impurity reduction estimator: homogeneous setting}\label{app:proof_impurity_homogeneous}
\begin{proof}[Proof of \Cref{th:impurity_iid}]
By \Cref{eq:impurity_decomposition}, the approximation error equals the
split heterogeneity correction exactly:
\begin{equation*}
    \left|\Delta I(j,t;\mathcal{S}_\nu) - \widehat{\Delta I}_{\mathrm{AvgImp}}(j,t;\mathcal{S}_\nu)\right|
= |\Delta\mathcal{E}(j,t;\nu)|,
\end{equation*}
with $\Delta\mathcal{E}(j,t;\nu)
:= \mathcal{E}(\nu) - \frac{n_L}{n_\nu}\mathcal{E}(\nu_L) - \frac{n_R}{n_\nu}\mathcal{E}(\nu_R),
$
so it suffices to show $\mathcal{E}(\nu) = O_p(K/n_\nu)$ for any node $\nu$.

\paragraph{Regression (variance impurity).}
Let $\mu_{\nu,k} := \frac{1}{n_{\nu,k}}\sum_{i \in \mathcal{S}_{\nu,k}} Y_i$
and $\mu_\nu := \frac{1}{n_\nu}\sum_{i \in \mathcal{S}_\nu} Y_i$
denote the local and global empirical means at node $\nu$, respectively.
The node heterogeneity is
$\mathcal{E}(\nu) = \sum_{k=1}^K \frac{n_{\nu,k}}{n_\nu}(\mu_{\nu,k} - \mu_\nu)^2 \ge 0$.
Under \Cref{as:iid}, each $\mu_{\nu,k}$ is an unbiased estimator of the
population mean $\mu := \mathbb{E}[Y]$, and so is $\mu_\nu$, giving
$\mathbb{E}[\mu_{\nu,k} - \mu_\nu] = 0$ and
$\mathrm{Var}(\mu_{\nu,k}) = O(1/n_{\nu,k})$ (finite by $\mathbb{E}[Y^2]<\infty$).
Taking expectations:
$$
\mathbb{E}\bigl[\mathcal{E}(\nu)\bigr]
= \sum_{k=1}^K \frac{n_{\nu,k}}{n_\nu}\,
  \mathbb{E}\bigl[(\mu_{\nu,k}-\mu_\nu)^2\bigr]
\le \sum_{k=1}^K \frac{n_{\nu,k}}{n_\nu}\,
  \mathrm{Var}(\mu_{\nu,k})
= \sum_{k=1}^K O\!\left(\frac{1}{n_\nu}\right) = O\!\left(\frac{K}{n_\nu}\right),
$$
where we used $\mathbb{E}[(\mu_{\nu,k} - \mu_\nu)^2] \le 2\,\mathrm{Var}(\mu_{\nu,k})
+ 2\,\mathrm{Var}(\mu_\nu) = O(1/n_{\nu,k})$.
Since $\mathcal{E}(\nu)\ge 0$, Markov's inequality gives $\mathcal{E}(\nu)=O_p(K/n_\nu)$.

\paragraph{Classification (Gini and entropy impurities).}
Let $\mathbf{p}_{\nu,k} := \bigl(p_{\nu,k}^{(c)}\bigr)_{c=1}^C$
and $\mathbf{p}_{\nu} := \bigl(p_{\nu}^{(c)}\bigr)_{c=1}^C$
be the local and global empirical class-probability vectors at node $\nu$,
where $p_{\nu,k}^{(c)} := \frac{1}{n_{\nu,k}}\sum_{i \in \mathcal{S}_{\nu,k}}\mathbf{1}[Y_i = c]$
and $p_{\nu}^{(c)} := \frac{1}{n_{\nu}}\sum_{i \in \mathcal{S}_{\nu}}\mathbf{1}[Y_i = c]$
for each class $c \in \{1,\ldots,C\}$.

For the \emph{Gini} index, $\mathcal{E}_{\mathrm{Gini}}(\nu)
= \sum_k \frac{n_{\nu,k}}{n_\nu}\|\mathbf{p}_{\nu,k}-\mathbf{p}_\nu\|^2$,
and since $\mathbb{E}\|\mathbf{p}_{\nu,k}-\mathbf{p}_\nu\|^2 = O(1/n_{\nu,k})$
by the same moment argument, $\mathcal{E}_{\mathrm{Gini}}(\nu)=O_p(K/n_\nu)$.

For \emph{entropy}, $\mathcal{E}_{\mathrm{Entropy}}(\nu)$ equals the
Jensen--Shannon divergence
$$
\mathrm{JSD}\bigl(\{\mathbf{p}_{\nu,k}\} \,\|\, \mathbf{p}_\nu\bigr)
:= \sum_{k=1}^K \frac{n_{\nu,k}}{n_\nu}
   D_{\mathrm{KL}}\bigl(\mathbf{p}_{\nu,k} \,\|\, \mathbf{p}_\nu\bigr),
\qquad
D_{\mathrm{KL}}(\mathbf{p} \,\|\, \mathbf{q})
:= \sum_{c=1}^C p^{(c)} \log \frac{p^{(c)}}{q^{(c)}}.
$$
For each $k$, a second-order Taylor expansion of $D_{\mathrm{KL}}(\mathbf{p}_{\nu,k}\|\mathbf{p}_\nu)$
around $\mathbf{p}_{\nu,k} = \mathbf{p}_\nu$ gives
$$
D_{\mathrm{KL}}\bigl(\mathbf{p}_{\nu,k} \,\|\, \mathbf{p}_\nu\bigr)
= \frac{1}{2}\sum_{c=1}^C \frac{(p_{\nu,k}^{(c)} - p_{\nu}^{(c)})^2}{p_{\nu}^{(c)}}
  + O\bigl(\|\mathbf{p}_{\nu,k} - \mathbf{p}_\nu\|^3\bigr)
= O\bigl(\|\mathbf{p}_{\nu,k} - \mathbf{p}_\nu\|^2\bigr),
$$
where the last equality holds provided $p_\nu^{(c)} \ge \epsilon > 0$ for all $c$
(so the denominator is bounded away from zero).
Hence $\mathcal{E}_{\mathrm{Entropy}}(\nu) = O_p(K/n_\nu)$ by the same moment
argument as for Gini.

\paragraph{Conclusion.}
The same bound applies to the child nodes $\nu_L$ and $\nu_R$
(which each receive a $\Theta(n_\nu)$ fraction of samples under a balanced-split
regularity condition), so
$|\Delta\mathcal{E}(j,t;\nu)| \le \mathcal{E}(\nu)
+ \frac{n_L}{n_\nu}\mathcal{E}(\nu_L) + \frac{n_R}{n_\nu}\mathcal{E}(\nu_R)
= O_p(K/n_\nu)$.\\
\qedhere
\end{proof}

\subsection{Impurity Decomposition}\label{app:proof_impurity_decomposition}
Proof of \Cref{th:impurity_decomposition} which states:
\begin{equation*}
    \Delta I(j,t;\mathcal S_\nu) = \sum_{k=1}^K \frac{n_{\nu,k}}{n_\nu} \Delta I(j,t;\mathcal S_{\nu,k}) + \Delta \mathcal E(j,t;\nu),
\end{equation*}
where the split-dependent heterogeneity term $\Delta \mathcal{E}(j,t;\nu)$ represents the change in the baseline node heterogeneity $\mathcal{E}(\nu)$ induced by the split $(j,t)$:
\[
\Delta \mathcal{E}(j,t;\nu) = \mathcal E(\nu) - \left( \frac{n_L(j,t)}{n_\nu} \mathcal E(\nu_L(j,t)) + \frac{n_R(j,t)}{n_\nu} \mathcal E(\nu_R(j,t)) \right).
\]
The baseline node heterogeneity $\mathcal{E}(\nu)$ term is defined as the Jensen gap of the local empirical target distributions:
\[
\mathcal{E}(\nu) = \Psi\left(\sum_{k=1}^K \frac{n_{\nu,k}}{n_\nu} \hat P_k\right) - \sum_{k=1}^K \frac{n_{\nu,k}}{n_\nu} \Psi(\hat P_k),
\]
where $\hat P_k$ is the empirical distribution of client $k$'s outcome at the node, and $\Psi$ is a standard statistical impurity function (e.g., variance, Gini index, or entropy--see \Cref{tab:sufficient_statistics}). 

For these split criteria, $\mathcal E(\nu) \ge 0$, with $\mathcal E(\nu)=0$ if and only if the local empirical outcome means are equal for all clients.

\begin{proof}
    We first establish the decomposition of the empirical impurity $I(\mathcal{S}_\nu)$ at a single node $\nu$.
    Let $\mathcal{S}_\nu$ be the set of samples at node $\nu$, and $\mathcal{S}_{\nu,k}$ the subset of samples belonging to client $k$, such that $\mathcal{S}_\nu = \bigcup_{k=1}^K \mathcal{S}_{\nu,k}$ and $\mathcal{S}_{\nu,k} \cap \mathcal{S}_{\nu,k'} = \emptyset$ for $k \neq k'$. Let $P_\nu$ denote the empirical distribution of the target variable $Y$ over the pooled sample $\mathcal{S}_\nu$, and $P_{\nu,k}$ denote the empirical distribution over the local sample $\mathcal{S}_{\nu,k}$. The empirical distribution assigns mass $1/n$ to each observation. Since the datasets are disjoint, the global empirical measure $P_\nu$ is the weighted average of the local measures. For any event $E$:
    \begin{align*}
        P_\nu(E) &= \frac{1}{n_\nu} \sum_{i \in \mathcal{S}_\nu} \mathbb{1}(y_i \in E)\\
        &= \sum_{k=1}^K \frac{n_{\nu,k}}{n_\nu} \left( \frac{1}{n_{\nu,k}} \sum_{i \in \mathcal{S}_{\nu,k}} \mathbb{1}(y_i \in E) \right)\\
        &= \sum_{k=1}^K w_{\nu,k} P_{\nu,k}(E),
    \end{align*}
    where $w_{\nu,k} = \frac{n_{\nu,k}}{n_\nu}$ represents the mixture weight of client $k$.

    The global empirical impurity is defined as a function $\Psi$ (mean squared error, Gini, cross-entropy, etc.) applied to the global distribution $P_\nu$:
    $$I(\mathcal{S}_\nu) = \Psi(P_\nu) = \Psi\left( \sum_{k=1}^K w_{\nu,k} P_{\nu,k} \right).$$
    We add and subtract the weighted sum of local impurities $\sum_{k=1}^K w_{\nu,k} \Psi(P_{\nu,k})$:
    \begin{align*}
        I(\mathcal{S}_\nu) &= \underbrace{\sum_{k=1}^K w_{\nu,k} \Psi(P_{\nu,k})}_{\text{Weighted Local Impurities}} + \underbrace{\left[ \Psi\left( \sum_{k=1}^K w_{\nu,k} P_{\nu,k} \right) - \sum_{k=1}^K w_{\nu,k} \Psi(P_{\nu,k}) \right]}_{\text{Jensen-gap / Heterogeneity Term } \mathcal{E}(\nu)}\\
        &= \sum_{k=1}^{K} \frac{n_{\nu,k}}{n_\nu} I(\mathcal{S}_{\nu,k}) + \mathcal{E}(\nu).
    \end{align*}

    Next, we apply this single-node decomposition to the impurity reduction $\Delta I(j,t;\mathcal S_\nu)$ of a candidate split $(j,t)$, which produces left and right children $\nu_L$ and $\nu_R$ of sizes $n_L$ and $n_R$. By definition:
    \begin{align*}
        \Delta I(j,t;\mathcal S_\nu) &= I(\mathcal S_\nu) - \frac{n_L}{n_\nu} I(\mathcal S_L) - \frac{n_R}{n_\nu} I(\mathcal S_R) \\
        &= \left( \sum_{k=1}^K \frac{n_{\nu,k}}{n_\nu} I(\mathcal S_{\nu,k}) + \mathcal{E}(\nu) \right) \\
        &\quad - \frac{n_L}{n_\nu} \left( \sum_{k=1}^K \frac{n_{L,k}}{n_L} I(\mathcal S_{L,k}) + \mathcal{E}(\nu_L) \right) \\
        &\quad - \frac{n_R}{n_\nu} \left( \sum_{k=1}^K \frac{n_{R,k}}{n_R} I(\mathcal S_{R,k}) + \mathcal{E}(\nu_R) \right).
    \end{align*}
    
    Rearranging the terms and noting that $\frac{n_L}{n_\nu}\frac{n_{L,k}}{n_L} = \frac{n_{L,k}}{n_\nu} = \frac{n_{\nu,k}}{n_\nu}\frac{n_{L,k}}{n_{\nu,k}}$ (and similarly for the right child), we can group the local impurities for each client:
    \begin{align*}
        \Delta I(j,t;\mathcal S_\nu) &= \sum_{k=1}^K \frac{n_{\nu,k}}{n_\nu} \underbrace{\left( I(\mathcal S_{\nu,k}) - \frac{n_{L,k}}{n_{\nu,k}} I(\mathcal S_{L,k}) - \frac{n_{R,k}}{n_{\nu,k}} I(\mathcal S_{R,k}) \right)}_{\Delta I(j,t;\mathcal S_{\nu,k})} \\
        &\quad + \underbrace{\mathcal E(\nu) - \left( \frac{n_L}{n_\nu} \mathcal E(\nu_L) + \frac{n_R}{n_\nu} \mathcal E(\nu_R) \right)}_{\Delta \mathcal E(j,t;\nu)}.
    \end{align*}
    This yields the decomposition in \eqref{eq:impurity_decomposition}.

    Finally, we show that for strictly concave impurities, 
\[
\mathcal E(\nu) = 0 \iff \forall (k,k'), \hat{\mathbb{E}}[Y\mid \mathcal S_\nu, H=k] = \hat{\mathbb{E}}[Y\mid \mathcal S_\nu, H=k'].
\]
    Let $\mu_{\nu,k} = \hat{\mathbb{E}}[Y \mid \mathcal S_\nu, H=k]$ denote the local empirical mean outcome (regression) or class probability vector (classification). 

    For regression, the Law of Total Variance identifies $\mathcal E(\nu)$ as the variance of these local means: 
    $\mathcal E(\nu) = \sum \frac{n_{\nu,k}}{n_\nu} (\mu_{\nu,k} - \mu_\nu)^2$.

    For classification, $\mu_{\nu,k}$ is the vector of class probabilities (since $\hat{\mathbb{E}}[Y_{\text{one-hot}}] = \mathbf{p}$), and impurity is defined as $\Psi(\mu_{\nu,k})$ where $\Psi$ is strictly concave (e.g., Gini or Entropy). 
    By Jensen's inequality, the gap $\mathcal E(\nu) = \Psi\left(\sum \frac{n_{\nu,k}}{n_\nu} \mu_{\nu,k}\right) - \sum \frac{n_{\nu,k}}{n_\nu} \Psi(\mu_{\nu,k})$ is non-negative. 

    In both cases, due to strict concavity, $\mathcal E(\nu)=0 \iff \mu_{\nu,1} = \dots = \mu_{\nu,K}$, i.e., the mean outcome or class distribution is identical across all clients.\\
    \qedhere
\end{proof}

\subsection{Exact Computation of CART Criteria from Local Statistics}
\begin{proof}[Proof of \Cref{thm:exact_impurity_reduction}]
We show that standard CART impurities can be computed exactly from additive client-side sufficient statistics, without centralizing raw samples. Let $\mathcal{S}_{\nu,k}$ be the samples held by client $k$ at node $\nu$, with $n_{\nu,k}=|\mathcal{S}_{\nu,k}|$. Each client reports a statistic vector $\mathbf{s}_{\nu,k}$, and the server forms
\begin{equation}
    \mathbf{s}_{\nu}=\sum_{k=1}^K \mathbf{s}_{\nu,k}.
\end{equation}
The impurity at node $\nu$ is then obtained by a deterministic map $\Psi(\mathbf{s}_{\nu})$, which we now detail for the following standard CART criteria.

\paragraph{Regression: MSE.}
Each client reports
\begin{equation}
    \mathbf{s}_{\nu,k}
    =
    \left(
    n_{\nu,k},
    \sum_{i\in\mathcal{S}_{\nu,k}} y_i,
    \sum_{i\in\mathcal{S}_{\nu,k}} y_i^2
    \right).
\end{equation}
After aggregation, write $\mathbf{s}_{\nu}=(n_{\nu},S_{\nu},Q_{\nu})$. The MSE impurity is
\begin{equation}
    \Psi(\mathbf{s}_{\nu})
    =
    \frac{Q_{\nu}}{n_{\nu}}
    -
    \left(\frac{S_{\nu}}{n_{\nu}}\right)^2.
\end{equation}

\paragraph{Classification: Gini and entropy.}
For $C$ classes, each client reports class counts
\begin{equation}
    \mathbf{s}_{\nu,k}
    =
    \left(
    n_{\nu,k},
    \sum_{i\in\mathcal{S}_{\nu,k}}\mathds{1}(y_i=1),
    \ldots,
    \sum_{i\in\mathcal{S}_{\nu,k}}\mathds{1}(y_i=C)
    \right).
\end{equation}
After aggregation, write $\mathbf{s}_{\nu}=(n_{\nu},N_{\nu,1},\ldots,N_{\nu,C})$. The Gini impurity is
\begin{equation}
    \Psi_{\mathrm{Gini}}(\mathbf{s}_{\nu})
    =
    1-\sum_{c=1}^C
    \left(\frac{N_{\nu,c}}{n_{\nu}}\right)^2,
\end{equation}
and the Shannon entropy impurity is
\begin{equation}
    \Psi_{\mathrm{Ent}}(\mathbf{s}_{\nu})
    =
    -
    \sum_{c=1}^C
    \frac{N_{\nu,c}}{n_{\nu}}
    \log_2
    \left(
    \frac{N_{\nu,c}}{n_{\nu}}
    \right),
\end{equation}
with zero-count terms omitted.

Thus, for MSE, Gini, and entropy, the global CART impurity is a deterministic function of summed local statistics. Therefore, split gains can be evaluated exactly from federated aggregates alone.
\end{proof}

\subsection{Top-L maximum gains aggregation recovers the true optimal feature}\label{app:topL_screening}

\begin{proposition}[Screening consistency under i.i.d. clients]\label{prop:topL_screening_appendix}

Consider a node $A$ and let $\Delta \mathcal I_\nu(j)$ denote the population (pooled) best gain for feature $j$. Let $j^*$ be the unique maximizer of $\Delta \mathcal I_\nu(\cdot)$ with margin $\gamma := \Delta \mathcal I_\nu(j^*) - \max_{j\neq j^*}\Delta \mathcal I_\nu(j) > 0$.

Assume that for client $k$, the empirical best-gain scores $M_{\nu,k}(j)$ concentrate around their population means with exponential tails (i.e., satisfy a sub-Gaussian tail bound). Then, there exist constants $C_1, C_2 > 0$ such that the probability of the optimal feature $j^*$ falling out of the top-$L$ local ranking decays exponentially with the local sample size $n_{\nu,k}$:
\begin{equation}
    \mathbb P(\mathrm{rank}_k(j^*) > L) \le C_1\, d\,\exp(-C_2\, n_{\nu,k}\, \gamma^2).
\end{equation}
\end{proposition}

\begin{proof}
The proof relies on decomposing the ranking failure into pairwise estimation errors and applying concentration bounds.

For $L \ge 1$, the event $\{\mathrm{rank}_k(j^*) > L\}$ implies that the optimal feature $j^*$ was empirically outperformed by at least $L$ other features, so $\{\mathrm{rank}_k(j^*) > L\} \subseteq \dots \subseteq \{\mathrm{rank}_k(j^*) > 1\}$, so we can bound the probability of the optimal feature not to be selected by:
$$  \mathbb P(\mathrm{rank}_k(j^*) > L) \leq \mathbb P\left( \bigcup_{j \neq j^*} \{ M_{\nu,k}(j) \ge M_{\nu,k}(j^*) \} \right).$$

Now consider a specific suboptimal feature $j\neq j^*$. By definition, the true gain gap is $\Delta \mathcal I_\nu(j^*) - \Delta \mathcal I_\nu(j) \ge \gamma$. For the empirical scores to reverse this order, i.e. satisfying $M_{\nu,k}(j) \ge M_{\nu,k}(j^*)$), their summed estimation errors for features $j$ and $j^*$ must exceed the margin $\gamma$. Specifically, let $\varepsilon_k(\cdot) = M_{\nu,k}(\cdot) - \Delta I_\nu(\cdot)$ be the estimation noise. The reversal condition implies:
$$  \varepsilon_k(j) - \varepsilon_k(j^*) \ge \Delta I_\nu(j^*) - \Delta I_\nu(j) \ge \gamma.$$

Finally, we have by the Triangle inequality 
\begin{align*}
    |\varepsilon_k(j)| + |\varepsilon_k(j^*)| \ge | \varepsilon_k(j) - \varepsilon_k(j^*)| \ge \gamma,
\end{align*}
so for the difference of errors to exceed $\gamma$, at least one of the individual error terms must have a magnitude of at least $\gamma/2$. By the concentration assumption (e.g., Hoeffding's inequality for bounded impurity functions), the probability of such a deviation for any single feature decays exponentially with the sample size:
$$\mathbb P\left( |\varepsilon_k(j)| \ge \frac{\gamma}{2} \right) \le c_1 \exp\left( -c_2\, n_{\nu,k} \left(\frac{\gamma}{2}\right)^2 \right).$$

Last, applying the union bound over the two error terms for the pair $(j, j^*)$, and subsequently over all $d-1$ possible suboptimal features $j$, we obtain:
$$\mathbb P(\mathrm{rank}_k(j^*) > L) \le \sum_{j \neq j^*}^d 2 c_1 \exp\left( -\frac{c_2}{4} n_{\nu,k} \gamma^2 \right).$$

Simplifying the constants yields the final bound $C_1 d \exp(-C_2 n_{\nu,k} \gamma^2)$.

\end{proof}

\newpage
\section{Impurity Decomposition under Covariate Shift}\label{app:imp_decomp_CS}

To evaluate a split under heterogeneity efficiently, the heterogeneity term $\mathcal{E}_\Phi$ defined in \Cref{th:impurity_decomposition} forbids the mere aggregation of local impurity gains. While one could compute $\mathcal{E}_\Phi(\mathcal{S}_\nu)$ explicitly and add it to the weighted local impurities, reconstructing the global impurity directly from sufficient statistics is more efficient. For standard impurity criteria, the global impurity $I(\mathcal{S}_\nu)$ dependents only on additive moments of the target variable (e.g., sums of targets or counts of classes). This allows for an efficient one-shot protocol: rather than sharing raw data, clients transmit a compact vector of sufficient statistics $\mathbf{s}_{\nu,k}$.

Because these statistics are linear, the server can reconstruct the global statistics $\mathbf{s}_{A}$ simply by summing the local vectors:
$\mathbf{s}_\nu = \sum_{k=1}^K \mathbf{s}_{\nu,k}.$
The server then evaluates the global centralized impurity using $\Psi(\cdot)$:
\begin{equation*}
I(\mathcal{S}_\nu) = \Psi(\mathbf{s}_\nu) = \sum_{k=1}^K \frac{n_{\nu,k}}{n_\nu} \Psi(\mathbf{s}_{\nu,k}) + \mathcal{E}_\Phi(A).
\end{equation*}

By computing $\Psi(\mathbf{s}_\nu)$ directly from the aggregated sums, the server implicitly captures the heterogeneity correction $\mathcal{E}_\Phi(A)$, ensuring the exact same split decision as a centralized algorithm.

The impurity reduction for a candidate split $(j,t)$ is finally computed as:$$\Delta I(j,t; \mathcal{S}_\nu) = \Psi(\mathbf{s}_\nu) - \left[ \frac{n_L}{n_\nu}\Psi(\mathbf{s}_L) + \frac{n_R}{n_\nu}\Psi(\mathbf{s}_R) \right].$$\Cref{tab:sufficient_statistics} details the minimal statistics $\mathbf{s}$ and the corresponding functional $\Psi$ for common loss functions.

\begin{table}[ht]
    \centering
    \caption{Definitions of impurity and sufficient statistics. The column $I(\mathcal{S})$ shows the standard centralized definition. The function $\Psi(\mathbf{s})$ computes this same value using only the aggregated statistics $\mathbf{s}$, allowing for exact federated evaluation.}
    \label{tab:sufficient_statistics}
    \renewcommand{\arraystretch}{2.8} 
    \resizebox{\textwidth}{!}{%
    \begin{tabular}{@{}llll@{}}
        \toprule
        \textbf{Objective} & \textbf{Centralized Impurity} $I(\mathcal{S}_\nu)$ & \textbf{Federated Impurity} $\Psi(\mathbf{s}_\nu)$ & \textbf{Sufficient Statistics} $\mathbf{s}_{\nu,k}$ \\ \midrule
        
        \textbf{Regression (MSE)} &
        $\displaystyle \frac{1}{n}\sum_{i \in \mathcal{S}_\nu} (y_i - \bar{y})^2$ &
        $\displaystyle \frac{S^{(YY)}_\nu}{n_\nu} - \left(\frac{S^{(Y)}_\nu}{n_\nu}\right)^2$ & 
        $\{ n_{\nu,k}, S^{(Y)}_{\nu,k}, S^{(YY)}_{\nu,k}\}$ \\

        \textbf{Classification (Gini)} &
        $\displaystyle 1 - \sum_{c=1}^C p_c^2$ &
        $\displaystyle 1 - \sum_{c=1}^C \left(\frac{n_{A,c}}{n_\nu}\right)^2$ & 
        $\{ n_{A,k,c} \}_{c=1}^C$ \\

        \textbf{Classification (Entropy)} &
        $\displaystyle - \sum_{c=1}^C p_c \log p_c$ &
        $\displaystyle - \sum_{c=1}^C \frac{n_{c}}{n} \log \left(\frac{n_{c}}{n}\right)$ & 
        $\{ n_{A,k,c} \}_{c=1}^C$ \\ 
        
        \bottomrule
    \end{tabular}%
    }
\end{table}
Standard impurity uses sample mean $\bar{y}$ and class probabilities $p_c$. In sufficient statistics: $S_{\nu,k}^{(Y)} = \sum_{i=1}^{n_{\nu,k}} y_i$, $S_{\nu,k}^{(YY)} = \sum_{i=1}^{n_{\nu,k}} y_i^2$, and $n_{A,k,c}= \sum_{i=1}^{n_{\nu,k}} \mathbb{1}(y_i = c)$.

Note that the heterogeneity term $\mathcal{E}(A)$ takes explicit forms for each impurity measure:
\begin{itemize}
    \item \textbf{Regression (MSE):} $\mathcal{E}(A) = \sum_{k=1}^K \frac{n_{\nu,k}}{n_\nu} \left( \bar{y}_{\nu,k} - \bar{y}_\nu \right)^2$
    \item \textbf{Classification (Gini):} $\mathcal{E}(A) = \sum_{k=1}^K \frac{n_{\nu,k}}{n_\nu} \sum_{c=1}^C \left( p_{A,k,c} - p_{A,c} \right)^2$
    \item \textbf{Classification (Entropy):} $\mathcal{E}(A) = \sum_{k=1}^K \frac{n_{\nu,k}}{n_\nu} \sum_{c=1}^C p_{A,k,c} \log \left( \frac{p_{A,k,c}}{p_{A,c}} \right)$.
\end{itemize}
However, in practice, the server does not need to compute these terms explicitly, as they are accounted for when evaluating the impurity $\Psi$ on the aggregated sufficient statistics.

\section{Exact Midpoint Enumeration}\label{app:exact_midpoints}

For completeness, we describe how centralized CART split candidates can be reproduced exactly in a federated setting.
At a given node $\nu$, centralized CART evaluates midpoints between consecutive unique feature values.
In principle, this can be achieved federatively by having each client transmit its node-level feature values, which the server merges into a globally sorted list.

In principle, this exact enumeration can be achieved in a federated architecture if every client $k \in \{1, \dots, K\}$ transmits its raw, local feature values for node $\nu$ to the server, which then merges and sorts them into a global list. Formally, for a given feature $j \in \{1, \dots, d\}$, let $\{v_{j,(r)}\}_{r=1}^{m_{\nu,j}}$ denote the sorted sequence of unique feature values pooled across all clients at node $\nu$, where $m_{\nu,j} \le n_\nu$ represents the total number of unique values. The exact set of split candidates evaluated by the server is then:
\begin{equation}
\mathcal{T}_\nu^{\mathrm{exact}} = \bigcup_{j=1}^d \left\{ (j,t) \;\middle|\; t = \frac{v_{j,(r)} + v_{j,(r+1)}}{2}, \quad 1 \leq r < m_{\nu,j} \right\}.
\end{equation}

This procedure exactly matches centralized CART but requires transmitting $O(n_\nu d)$ floating-point values per node and produces up to $d(n_\nu-1)$ candidates.
As such, it is impractical beyond small-scale settings and is not used in our method.

\section{Experimental Setting and Additional Simulations}\label{sec:add_sims}

\textbf{Definition of the regression function $f$.} To define a complex, nonlinear, tree-like ground-truth $f:\mathbb R^d\to\mathbb R$, we distill a tree on an auxiliary dataset sampled over a fixed domain. We draw $N_\mathrm{aux}=10,000$ samples, $x_\mathrm{aux}\sim \mathrm{Unif}([-10,10]^d)$ with $d=20$, and rescale them to $\tilde{x} \in [0,1]^d$. We define the auxiliary target as $y_\mathrm{aux}=\sum_{j=1}^d \psi_{(j \bmod 3)} (\tilde x_j, \tilde x)$, using three component functions: $\psi_0(z)=\mathds{1}\{z>0.5\}$ (step), $\psi_1(z)=\sin(4\pi z)$ (frequency), and $\psi_2(z, \tilde x)=\tilde x_j \times \tilde x_{(j \bmod d) + 1}$ (interaction). Finally, we fit a CART regressor of maximum depth $8$ on $(x_{\mathrm{aux}},y_{\mathrm{aux}})$ and define the ground truth $f(\cdot)$ as the resulting outcome function.

\subsection{Covariate shift: when local objectives mis-rank global splits}\label{subsec:cov_shift_sims}
We study two mechanisms by which covariate shift can affect federated split selection. Throughout, \textit{FedForest Quantiles} remains stable because it constructs candidates from a pooled quantile sketch and evaluates them with exact pooled impurities, whereas \textit{FedForest AvgImp} can fail because it optimizes a different objective (the average of local gains).

\subsubsection{Disjoint supports}\label{subsubsec:islands}
We generate $K=2$ clients with $n_k=150$ samples each and features in $\mathbb R^d$ (here $d=5$). Client $k$ draws
\[
X_i^{(k)}\sim \mathcal N(\mu_k, I_d),\qquad 
\mu_k = \big((-1)^k\gamma,0,\ldots,0\big),
\]
so that increasing $\gamma\ge 0$ creates a growing gap between client supports along the first coordinate, effectively moving the pooled data distribution of $(X)$ from i.i.d. samples \Cref{as:iid} when $\gamma=0$ to a (Gaussian here) bimodal mixture with $\rho_1=\rho_2=\frac{1}{2}$. Outcomes follow a step rule on $X^{(1)}$:
\[
Y_i = f(X_i)+\varepsilon_i,\qquad 
f(x)=10\times \mathbb 1\{x^{(1)}>0\},\qquad 
\varepsilon_i\sim\mathcal N(0,\sigma^2),
\]
with $\sigma=1$, so the irreducible MSE is $1$. All RF methods are trained with a maximum depth of $8$ and $50$ trees. Table~\ref{tab:disjoint_supports_covariate_shift} shows that as $\gamma$ increases, FedForest {AvgImp} fails to select the pooled-data optimal split $(j=1,t=0)$, while quantile candidates remain robust and track centralized training.

\begin{table}[h] \centering \caption{Mean Squared Error (MSE) comparison under Ranking Disagreement (Lower is Better). At higher shift scales ($\delta \ge 2.5$), the local greedy methods (AvgImp) fail to identify the optimal split in the gap between client supports, while global quantile methods remain robust.} \label{tab:disjoint_supports_covariate_shift} \resizebox{\textwidth}{!}{%
\begin{tabular}{l c c c c c c} \toprule \textbf{Shift ($\gamma$)} & \textbf{FedForest AvgImp (X)} & \textbf{FedForest AvgImp (X,H)} & \textbf{FedForest Quantiles (X)} & \textbf{FedForest Quantiles (X,H)} & \textbf{Centralized (X)} & \textbf{Centralized (X,H)} \\ \midrule 0.0 & $1.22 \pm 0.18$ & $1.22 \pm 0.18$ & $1.22 \pm 0.22$ & $1.22 \pm 0.22$ & $1.22 \pm 0.20$ & $1.22 \pm 0.21$ \\ 1.0 & $1.26 \pm 0.22$ & $1.26 \pm 0.22$ & $1.20 \pm 0.18$ & $1.20 \pm 0.18$ & $1.19 \pm 0.17$ & $1.19 \pm 0.17$ \\ 2.5 & $5.86 \pm 5.36$ & $5.86 \pm 5.36$ & $1.16 \pm 0.12$ & $1.16 \pm 0.12$ & $1.19 \pm 0.17$ & $1.22 \pm 0.15$ \\ 5.0 & $19.22 \pm 3.99$ & $19.22 \pm 3.99$ & $1.05 \pm 0.04$ & $1.05 \pm 0.04$ & $1.06 \pm 0.04$ & $1.06 \pm 0.04$ \\ \bottomrule \end{tabular}%
} \end{table}

\subsubsection{Overlapping supports: failure driven by the heterogeneity term}\label{subsubsec:hetero_term}
{AvgImp} can also be deceiving with overlapping supports, when the pooled impurity differs from the weighted average of local impurities due to the heterogeneity term $\mathcal E$ in Theorem~\ref{th:impurity_decomposition}.

To illustrate this, we simulate $K=2$ clients with features drawn as 
\[
X_i^{(k)} \sim \mathcal{N}(\mu_k, \sigma_X^2), \qquad \mu_k \in \{-\gamma, +\gamma\}, \qquad \sigma_X = 1.5 I_d,
\]
and a shared conditional model that is linear in the first feature:
\[
Y_i = X_i^{(1)} + \varepsilon_i, \qquad \varepsilon_i \sim \mathcal{N}(0, \sigma^2).
\]

Under balanced mixture weights, the true global optimal split lies near the pooled population mean ($\E[X] = 0$), a region well-populated by both clients. However, as the shift $\gamma$ increases, the heterogeneity term $\mathcal{E}(\nu)$ dominates. Consequently, maximizing the averaged local impurity gain incorrectly favors thresholds near the individual client means ($\pm\gamma$) rather than the global optimum.

Table~\ref{tab:optimization_valley} confirms this effect. While \textit{FedForest Quantiles} consistently tracks the centralized baselines across all shift scales, \textit{AvgImp} degrades sharply as $\gamma$ grows, proving that averaging local impurities is inherently vulnerable to covariate shift, regardless of support overlap.

\begin{table}[h]
\centering
\caption{MSE comparison of methods with covariate shift.}
\label{tab:optimization_valley}
\resizebox{\textwidth}{!}{%
\begin{tabular}{l c c c c c c}
\toprule
\textbf{Shift ($\gamma$)} & \textbf{FedForest AvgImp (X)} & \textbf{FedForest AvgImp (X,H)} & \textbf{FedForest Quantiles (X)} & \textbf{FedForest Quantiles (X,H)} & \textbf{Centralized (X)} & \textbf{Centralized (X,H)} \\
\midrule
0.0 & $1.22 \pm 0.18$ & $1.22 \pm 0.18$ & $1.22 \pm 0.22$ & $1.22 \pm 0.22$ & $1.22 \pm 0.20$ & $1.22 \pm 0.21$ \\
1.0 & $1.26 \pm 0.22$ & $1.26 \pm 0.22$ & $1.20 \pm 0.18$ & $1.20 \pm 0.18$ & $1.19 \pm 0.17$ & $1.19 \pm 0.17$ \\
2.5 & $5.86 \pm 5.36$ & $5.86 \pm 5.36$ & $1.16 \pm 0.12$ & $1.16 \pm 0.12$ & $1.19 \pm 0.17$ & $1.22 \pm 0.15$ \\
5.0 & $19.22 \pm 3.99$ & $19.22 \pm 3.99$ & $1.05 \pm 0.04$ & $1.05 \pm 0.04$ & $1.06 \pm 0.04$ & $1.06 \pm 0.04$ \\
\bottomrule
\end{tabular}%
}
\end{table}

\subsection{Covariate and outcome shifts}\label{subsec:cov_out_shift}
We consider the fully heterogeneous regime of \Cref{eq:full_hetero}, where both covariates ($\alpha_k=.5$ and $\gamma_k=3$) and outcomes ($\delta=1.5$) vary across clients. Similar hyperparameters are used as in \Cref{sec:simulations}.

\begin{figure}[h]\label{fig:fully_hetero}
    \centering
    \includegraphics[width=.80\linewidth]{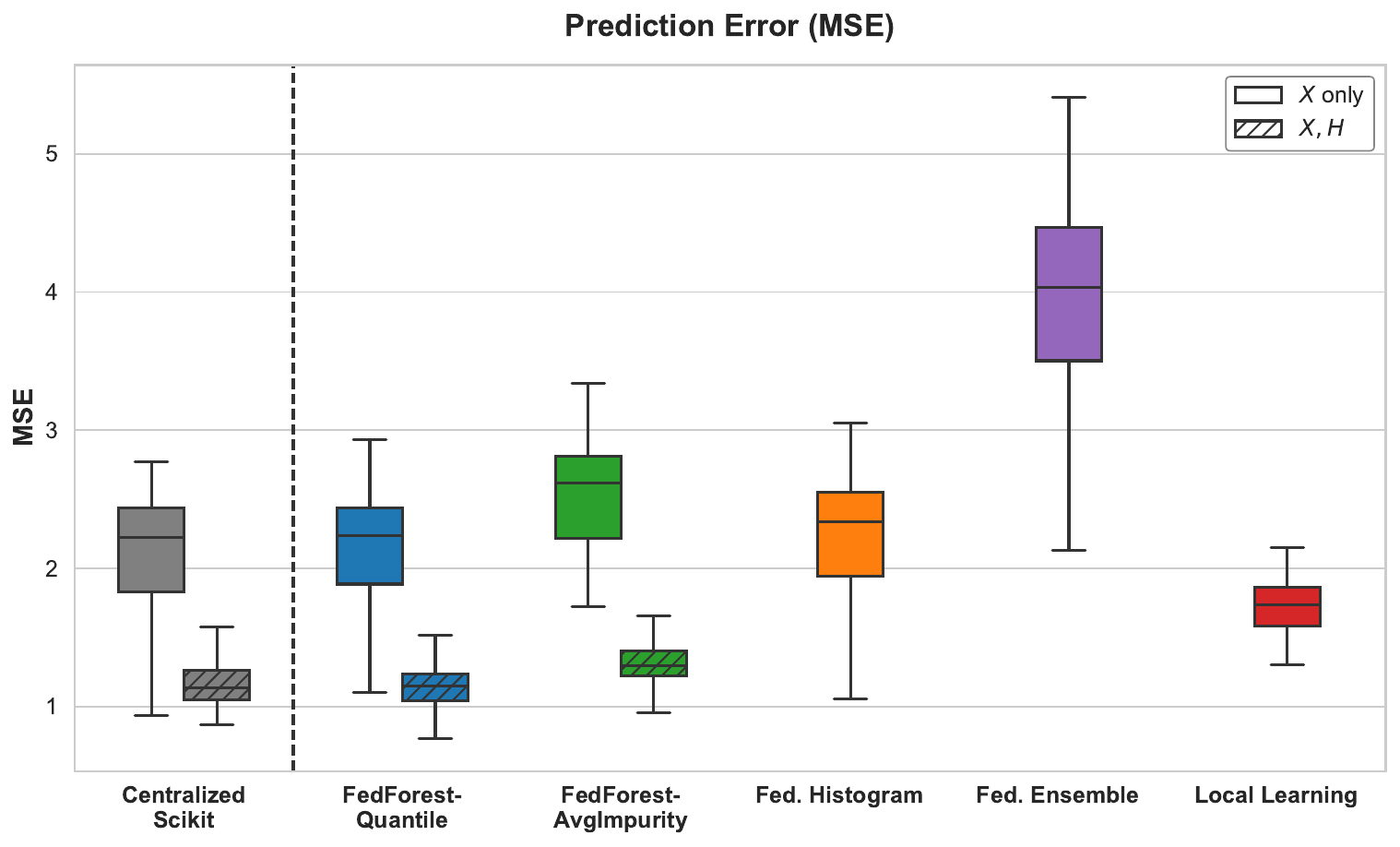}
    \caption{Methods comparison on fully heterogeneous clients (\Cref{eq:full_hetero}).}
    \label{fig:fully_heterogeneous_benchmark}
\end{figure}

The results align with previous conclusions from the covariate shift and outcome shift paragraphs in \Cref{sec:simulations}.

\subsection{Robustness to outliers}\label{subsec:outliers}

To assess robustness to outliers, we use the same data-generating process as in the homogeneous (\textit{i.i.d.}) simulation setting described in \Cref{sec:simulations}. We then contaminate the training data by introducing outliers at fractions ranging from $0\%$ to $20\%$ of the total sample size. For the contaminated observations, the corresponding feature vectors are multiplied by a factor of $1000$, producing extreme covariate values while leaving the remaining data unchanged.

\begin{table}[hbt!]
\centering
\caption{Model Degradation (MSE) vs. Outlier Fraction. Results show the mean $\pm$ standard deviation across $20$ simulations.}
\label{tab:outlier_degradation}
\resizebox{\textwidth}{!}{%
\begin{tabular}{l cccccc}
\toprule
 & \multicolumn{6}{c}{\textbf{Outlier Fraction}} \\
\cmidrule(lr){2-7}
\textbf{Method} & \textbf{0\%} & \textbf{2\%} & \textbf{5\%} & \textbf{10\%} & \textbf{15\%} & \textbf{20\%} \\
\midrule
Centralized Scikit X             & $0.48 \pm 0.05$ & $0.48 \pm 0.05$ & $0.48 \pm 0.07$ & $0.52 \pm 0.10$ & $0.56 \pm 0.09$ & $0.53 \pm 0.08$ \\
Fed. Average Impurity X (ours)   & $0.51 \pm 0.12$ & $0.55 \pm 0.09$ & $0.63 \pm 0.08$ & $0.64 \pm 0.11$ & $0.65 \pm 0.11$ & $0.66 \pm 0.09$ \\
Fed. Quantiles X (ours)          & $0.46 \pm 0.08$ & $0.48 \pm 0.07$ & $0.55 \pm 0.08$ & $0.54 \pm 0.09$ & $0.59 \pm 0.09$ & $0.56 \pm 0.08$ \\
Fed. Histogram (Kalloori et al.) & $0.51 \pm 0.11$ & $0.53 \pm 0.08$ & $0.67 \pm 0.15$ & $0.90 \pm 0.18$ & $1.09 \pm 0.32$ & $1.55 \pm 0.21$ \\
Fed. Ensemble (Hauschild et al.) & $1.10 \pm 0.37$ & $1.08 \pm 0.23$ & $1.33 \pm 0.24$ & $1.30 \pm 0.32$ & $1.47 \pm 0.31$ & $1.41 \pm 0.17$ \\
Local Learning                   & $1.27 \pm 0.41$ & $1.22 \pm 0.24$ & $1.45 \pm 0.27$ & $1.44 \pm 0.34$ & $1.61 \pm 0.35$ & $1.55 \pm 0.21$ \\
\bottomrule
\end{tabular}%
}
\end{table}

The results in \Cref{tab:outlier_degradation} show that the proposed methods are substantially more stable under outlier contamination than the histogram-based federated baseline. In particular, the MSE of Fed. Histogram increases sharply as the outlier fraction grows, rising from $0.51$ with no contamination to $1.55$ at $20\%$ contamination. By contrast, Fed. Quantiles remains close to the centralized reference across all contamination levels, with MSE increasing only from $0.46$ to $0.56$. Fed. Average Impurity also degrades gradually, reaching $0.66$ at the highest contamination level.

Histogram-based methods are sensitive to outliers because global min--max binning can stretch the feature range, collapsing most clean samples into coarse bins and reducing split quality. In contrast, quantile-based methods rely on ranks, so extreme values are confined to tail bins while the remaining bins preserve resolution over the clean data. This accounts for their milder degradation under outlier contamination.

\subsection{Heterogeneity diagnostics for adaptive optimization}

To balance communication efficiency with predictive performance, we run a lightweight diagnostic step at the root node of the first tree. The goal is to decide whether i.i.d.-specific accelerations (top-$L$ screening, local-gain averaging, and client subsampling) are appropriate, or whether we should default to the fully robust mode (exact split evaluation and quantile candidates).

\paragraph*{Covariate shift test.}

We test for covariate shift by assessing the null hypothesis $H_0: X \indep H$ using a federated discriminator of clients. Instead of estimating density ratios—either parametrically or via kernel methods with high communication cost—we reuse the impurity-based split selection of federated decision trees by treating the client index $H$ as the target variable. A federated Random Forest is trained to predict $H|X$ using standard CART criteria.

Under $H_0$, the feature distributions $P(X|H=k)$ are identical across clients, implying that no split $\theta$ can separate samples by site better than random chance. We quantify this using the site impurity gain. For a $(j,t)$ candidate split at a node $\nu$, each client $k$ sends the count of its local samples falling into the left partitions: $n_{L,k} = \sum_{i \in \mathcal{S}_{\nu,k}} \mathbb{1}(x_{i}^{(j)} \le t)$, enough to derive $n_{R,k}=n_{\nu,k}-n_{L,k}$.
The server then computes the $K$-class impurity gain for the site variable $H$:
$$\Delta {I}_H(j,t;\mathcal{S}_\nu)\!\! = \!I_{H}(\mathcal{S}_\nu) - \left( \frac{n_L}{n_\nu}I_H(\mathcal{S}_L) + \frac{n_R}{n_\nu}I_H(\mathcal{S}_R) \right),$$
where $I_H(\mathcal{S}_L) = 1 - \sum_{k=1}^K \left( \frac{n_{L,k}}{n_L} \right)^2$  (Gini).
Note that although each client contains only a single "class" of site labels--i.e., the local impurity is zero--the aggregation of local counts yields the exact global impurity. We thus reject $H_0$ whenever the maximal root gain or the predictive performance of the site classifier (e.g., one-vs-rest AUC) exceeds a prescribed threshold. This non-parametric technique is federated and further provides interpretable insights on which features are heterogenously distributed across sites.

\paragraph*{Outcome Heterogeneity Test.}\label{test:outcome_shift}
To detect outcome heterogeneity driven by site effects, one could formally test the null hypothesis that the conditional expectation of the target is invariant across sites. This involves performing $K-1$ pairwise comparisons (e.g., against a reference site $k^*$):$$H_0: \mathbb E[Y\mid X,H=k]=\mathbb E[Y\mid X,H=k'].$$While statistical procedures such as permutation tests, Kruskal-Wallis tests on residuals, or cosine similarity tests can be adapted to the federated setting to provide formal significance levels, they typically incur high communication costs and require delicate threshold calibration. Instead, we propose a pragmatic model-based alternative: we compare the validation performance (e.g., accuracy or $R^2$) of two federated Random Forests—one trained solely on features $X$, and an augmented model trained on $(X, H)$. A significant performance gain by the augmented model implies that the site index $H$ contains predictive signal not captured by $X$, effectively acting as a proxy for rejecting $H_0$.